\documentclass{article}

\usepackage{arxiv}

\usepackage[utf8]{inputenc} 
\usepackage[T1]{fontenc}    
\usepackage{hyperref}       
\usepackage{url}            
\usepackage{booktabs}       
\usepackage{amsfonts}       
\usepackage{nicefrac}       
\usepackage{microtype}      
\usepackage{lipsum}

\usepackage{makecell}
\usepackage{multirow}
\usepackage{amsmath}
\usepackage{amsthm}
\usepackage{graphicx}
\usepackage{subcaption}
\usepackage{cite}
\usepackage{authblk}

\newtheorem{lemma}{Lemma}
\newtheorem{theorem}{Theorem}
\newtheorem{protocol}{Protocol}
\newtheorem{corollary}{Corollary}
\providecommand{\keywords}[1]{\textbf{\textit{Keywords:}} #1}

\graphicspath{{figures/}} 
\title{Anomaly detection with superexperts under delayed feedback}

\author[1]{Raisa Dzhamtyrova}
\author[1, 2] {Carsten Maple }

\affil[1]{\footnotesize The Alan Turing Institute}
\affil[2]{\footnotesize WMG Cyber Security Centre, University of Warwick}
\affil[ ] {\textit \url{rdzhamtyrova@turing.ac.uk} \hspace*{0.3cm} \url{cm@warwick.ac.uk}}
\begin{document}
\maketitle

\begin{abstract}
	The increasing connectivity of data and cyber-physical systems has resulted in a growing number of cyber-attacks. Real-time detection of such attacks, through the identification of anomalous activity, is required so that mitigation and contingent actions can be effectively and rapidly deployed. We propose a new approach for aggregating unsupervised anomaly detection algorithms and incorporating feedback when it becomes available. We apply this approach to open-source real datasets and show that both aggregating models, which we call experts, and incorporating feedback significantly improve the performance. An important property of the proposed approaches is their theoretical guarantees that they perform close to the best superexpert, which can switch between the best performing experts, in terms of the cumulative average losses.
\end{abstract}


\section{Introduction}
Anomaly detection is an important and well-studied problem that has applications in various domains. Anomalies are defined to be patterns in data that differ from the expected behaviour \cite{chandola2009survey}. Early detection of data breaches, fraudulent activities, cyber-attacks, and system failures can potentially prevent adverse events and financial and reputational losses. The optimal choice of algorithm depends on the type of problem and the nature of the anomalies.~Anomaly detection techniques can be divided into three groups: supervised, in which anomalies are labelled; semi-supervised, in which only the normal class is available; and unsupervised, in which there are no labels available. The nature of anomalies can also be classified into three groups: point anomalies, where a single point can be anomalous; contextual anomalies, where a data instance is anomalous in a specific context; and collective anomalies, where a collection of related data instances is anomalous concerning the entire dataset \cite{chandola2009survey}. The paper \cite{chandola2009survey} provides an extensive survey of anomaly detection techniques for intrusion detection, fraud detection, fault detection, and other domains. The survey of novelty detection methods that aim to identify test data, which is in some respect is different from the available training data can be found in \cite{chandola2009comparative}. The article \cite{pimentel2014novelty} gives a comparative study and evaluation of a large number of anomaly detection techniques on artificial and publicly available datasets. Furthermore, recent papers \cite{chalapathy2019survey, habeeb2019survey} provide surveys of deep learning methods and real-time big data processing techniques for anomaly detection.

In this paper, we propose a new approach to aggregating unsupervised anomaly detection algorithms, which takes into account feedback, when it becomes available. It means, that we need to make several predictions before we see the outcomes which are revealed later altogether. Though most of the anomaly detection algorithms are unsupervised because feedback is often not available, it is realistic to assume that we will know which observations were anomalous after some time. For example, if there was a system failure, or the analysts monitor the system weekly/monthly and manually label the anomalous points. The anomalous class is usually a small percent of total observations, and therefore one needs only provide timestamps of anomaly events instead of providing labels for all data points. 

The proposed approach is based on the Prediction with Expert Advice (PEA) framework.  In PEA a {\em learner} competes with {\em experts}, which may be human experts, other predictive models or even classes of functions.
The application of PEA to solve real-world problems includes domains such as sport \cite{vovk_fedor_brier}, renewable energy forecasting \cite{dzhamtyrova2020quantile}, and finance \cite{CovOrd96, dzhamtyrova2020var}. The theory goes back to
\cite{maj_little_1994, cbl_book_prediction}. The investment methods competitive with large parametric classes are introduced in the universal portfolios theory in \cite{CovOrd96}, which compete against portfolio selection techniques. In the PEA framework, the goal of the learner is to construct a {\em prediction strategy} which can predict almost as well as the best expert in terms of cumulative losses. This approach works in online mode, where one does not have a training dataset to find the best model in advance. It should be noted that the performance of models changes with time. As a result, reliance on one model can be dangerous in practice since the best model can degrade with time and produce unreliable predictions. Instead of re-training a model periodically, we propose a different approach that mixes predictions of several models based upon their current performance.

The Aggregating Algorithm (AA) was introduced in \cite{vovk_aggr} to solve the problem of competitive prediction. The method assigns initial weights to experts and at each step updates their weights based on experts' current losses. The approach is a generalisation of the Bayesian mixture of probabilistic models and coincides with the Bayesian mixture in the case of the logarithmic loss function \cite{vovk_cols}. However, the AA works with a wide range of functions \cite{vovk_cols}, such as square-loss, which we consider in this paper alongside logarithmic loss. The AA provides a theoretical guarantee that in the case of a finite number of experts a learner who follows the strategy of the AA has the loss as small as the best expert's loss plus a constant at any time step in the future. The generalisation of the AA for the case of predicting of {\em packs}, vector-valued outcomes, is proposed in \cite{Adamskiy2019Packs}. The authors show that the AA can be applied to vector-valued outcomes and its theoretical guarantees are similar to those for the scalar outcomes.

In this paper, the term {\em superexpert} \cite{kalnishkan2018ml_lectures} will refer to a strategy that can `switch' between experts. Commonly, some models can provide good predictions in particular areas of their expertise. For example, in a time-series forecasting problem some models are good at detecting seasonality, and others at finding trends. In the case of anomaly detection, some models can be good at predicting system failures whereas others are good at forecasting security breaches. In practice, it is important to identify different types of anomalies. However, most models would perform well only at some particular tasks. Furthermore, some models provide better predictions at different time segments of data. We assume that the superexpert can predict as a particular expert at some time interval and then switch to another expert. Our goal is to find a strategy that can perform close to these superexperts. The AA performs as well as the best expert in terms of cumulative losses. The Fixed-share and Variable-share algorithms proposed in \cite{warmuth_tracking_original} are generalisations of the AA and can asymptotically perform close to the best superexpert. 

In this paper, we propose an adaptation of the Fixed-share and Variable-share algorithms that can work with the delayed feedback. These approaches are also an extension of the theory of Aggregating Algorithm for Predictions of Packs \cite{Adamskiy2019Packs}. A similar setting of the online learning under delayed feedback is investigated in \cite{korotin2020adaptive}, where an adaptation of the Hedge algorithm \cite{freund1997hedge} under delayed feedback is proposed. In \cite{vyugin2018tracking} an adaptation of the Hedge algorithm for tracking the best combination of experts is investigated. The delayed setting also has been studied within the online convex optimisation framework \cite{joulani2013online, quanrud2015online}. We prove the worst-case upper bounds on the performance of the proposed approaches that are similar to those in \cite{warmuth_tracking_original}, however, they are of a different kind: the bounds provide guarantees on the cumulative average losses instead of the cumulative losses. 
In the case when the number of switches is $k$, the total number of steps is $T$, and the number of experts is $N$, the additional average loss of the Fixed-share algorithm is bounded by $O(k \log N + k \log (T /k))$. When the loss per trial can be bounded, Variable-share produces an additional average loss bounded by $O(k \log N + k \log(\hat{L} / k))$, where $\hat{L}$ is the total average loss of the best superexpert with $k$ switches and $N$ base experts. In comparison to Fixed-share, the bound of Variable-share does not contain the number of steps $T$. The bounds of these algorithms are worse than the AA, which is unsurprising since they compete with all superexperts, and the number of superexperts grows exponentially with the number of switches. The algorithms and the theoretical bounds are generic and can be applied to a wide range of tasks and a wide range of loss functions.

We then apply these approaches to the problem of anomaly detection. It provides a new method for aggregating unsupervised anomaly detection algorithms that take into account true labels when they become known. To the best of our knowledge, it is the first approach for anomaly detection, which aggregates different algorithms and automatically takes into account delayed feedback. In the case of anomaly detection, it is also relatively low cost to collect the true labels as the anomalies' percent is usually low. In addition, the methods are highly explainable as they are based on the weights update that reflects a real-time assessment of the algorithm's performance. We test these approaches on the Numenta Anomaly Benchmark (NAB) which contains both artificial and real datasets with labelled anomalies~\footnote{\label{nab} \url{https://github.com/numenta/NAB}}, which provides an environment for testing anomaly detection algorithms on streaming data. The data contains a mix of contextual and collective anomalies, clean and noisy data, artificial and real scenarios, and data streams where the statistics evolve over time. The framework currently contains 15 unsupervised anomaly detection algorithms. The experimental results show a significant performance improvement compared to any single model in terms of various classification metrics and loss functions. An extension of the theory of Aggregating Algorithm for Prediction of Packs to Fixed-share and Variable-share also provide a significant performance improvement. Though the proposed methods have theoretical guarantees on the cumulative average losses, the experiments show that they also work well for the cumulative losses. We provide an open-source code of our implementations to enable reproduction of our results~\footnote{\label{anomaly} \url{https://github.com/alan-turing-institute/anomaly_with_experts}}.

\section{Contributions}
Our first contribution is an adaptation of the Fixed-share and Variable-share algorithms to delayed feedback, which extends the theory of Aggregating Algorithm for Predictions of Packs. 
We prove the theoretical guarantees of these approaches, which show that they can predict asymptotically as well as any superexpert in terms of the cumulative average losses. The algorithms and the theoretical bounds are generic and can be applied to a wide range of tasks, such as classification and regression, and a wide range of loss functions.

The second contribution is a new approach for aggregating unsupervised anomaly detection algorithms and incorporating feedback, which comes with a delay. To the best of our knowledge, it is the first approach for anomaly detection, which aggregates different algorithms and automatically takes into account delayed feedback. In the case of anomaly detection, it is also easy to incorporate feedback as the percent of the anomalies is usually small. Therefore, one can retrospectively indicate only the timestamps with anomalies, without the need to provide all the data points. We formulate the problem and introduce the game and loss functions. Apart from the theoretical guarantees, another advantage of the approach is that it is highly explainable: the weights' update reflects the current algorithms' performances.

The third contribution is an application of our approach on an open-source framework with a large variety of both real and artificial time-series. The results show that both aggregating models and incorporating feedback can bring a significant improvement to the performance of unsupervised anomaly detection algorithms in terms of various classification metrics and loss functions. The approach also provides a significant improvement compared to the Aggregating Algorithm for Pack Averages (AAP-current). Though the proposed approaches have theoretical guarantees on the cumulative average losses, the experimental results show that they also perform well in terms of cumulative losses.

\section{Framework} \label{sec:framework}
Suppose that we have access to the predictions of some pool of experts that predict outcomes from some outcome space $\Omega$. The experts' predictions come from some decision space $\Gamma$. In the PEA framework a learner's goal is to construct a {\em prediction strategy} whose performance will be close to the performance of the best expert from the pool in terms of a specified loss function $\lambda: 
\Omega \times \Gamma \rightarrow \mathbb{R}$. A triple $\mathcal{G} = \langle \Omega, \Gamma, \lambda \rangle$ is called a {\em game}. At each step $t$ the learner outputs its predictions after seeing the predictions of experts. In the classic online learning protocol, the true outcome comes immediately after making the prediction. In this paper, we modify the classical protocol for the delayed feedback, when the outcomes may come after an arbitrary number of steps after making the prediction. The following is the protocol of prediction with expert advice under the delayed feedback, with the notations explained in the following paragraph.
\begin{protocol}[Prediction with expert advice under delayed feedback]~
	\label{protocol_delayed} 
	\begin{tabbing} 
		$L^{\text{average}}_0 := 0$ \\
		$L^{\text{average}}_0(\mathcal{E}_i) := 0$ \\
		\quad\=\quad\=\quad\=\quad\=\quad\kill
		FOR $t=1,2,\ldots$\\
		\>\> FOR $d=1,2,\ldots D_t$\\
		\>\>\> Experts $\mathcal{E}_i$ output $\xi_{t,d}(i) \in \Gamma,~i =1, 2, \dots, N$\\
		\>\>\> Learner outputs $\gamma_{t,d} \in \Gamma$\\
		\>\> END FOR \\
		\>\> Nature announces $y_{t, d} \in \Omega $,~$d=1,\ldots D_t$\\
		\>\> $L^{\text{average}}_t := L^{\text{average}}_{t-1} + \frac{1}{D_t}\sum_{d=1}^{D_t}\lambda(y_{t,d}, \gamma_{t,d})$ \\
		\>\> $L^{\text{average}}_t(\mathcal{E}_i) := L^{\text{average}}_{t-1}(\mathcal{E}_i) + \frac{1}{D_t}\sum_{d=1}^{D_t} \lambda(y_{t,d}, \xi_{t,d}(i)),~ i =1, 2, \dots, N$ \\
		END FOR
	\end{tabbing}
\end{protocol}

At each step $t$ and each delay $d$ expert $\mathcal{E}_i$ outputs its prediction $\xi_{t,d}(i) \in \Gamma, ~i = 1, 2, \dots, N$, where $N$ is the number of experts. After seeing all experts' predictions, the learner outputs prediction $\gamma_{t, d} \in \Gamma$. After the learner made $D_t$ predictions, nature announces true outcomes $y_{t, d} \in \Omega $, $d=1,2,\dots, D_t$, the experts and the learner suffer losses $\lambda(y_{t,d}, \xi_{t,d}(i))$, $d=1,2,\dots, D_t$ and $\lambda(y_{t,d}, \gamma_{t,d})$, $d=1,2,\dots, D_t$ respectively.
In the classical protocol, the aim of the learner is to keep its cumulative loss small compared to the cumulative losses of all the experts at each step. In this paper, we 
construct algorithms that can compete in terms of the {\em cumulative average losses}, where the cumulative average losses of the learner and expert $\mathcal{E}_i$ are defined as:
\begin{equation} \label{eq:cum_avg_learner}
	L^{\text{average}}_T := \sum_{t=1}^T \frac{1}{D_t}\sum_{d=1}^{D_t}\lambda(y_{t,d}, \gamma_{t,d}),
\end{equation}	
\begin{equation} \label{eq:cum_avg_expert}
	L^{\text{average}}_T(\mathcal{E}_i) := \sum_{t=1}^T\frac{1}{D_t}\sum_{d=1}^{D_t} \lambda(y_{t,d}, \xi_{t,d}(i)),~ i =1, 2, \dots, N.
\end{equation}	
We define a {\em regret} or an {\em additional average loss} to be the difference between the cumulative average losses of the best expert and the learner $R_t = L^{\text{average}}_t - \min_{i = 1, \dots, N} L^{\text{average}}_t(\mathcal{E}_i) $. 

We construct our algorithms following the theory of Aggregating Algorithm for prediction of {\em packs}, vector-valued outcomes \cite{Adamskiy2019Packs}. In this theory, three algorithms were proposed: Aggregating Algorithm for Packs with the Known Maximum (AAP-max), Aggregating Algorithm for Packs with the Unknown Maximum (AAP-incremental), and Aggregating Algorithm for Pack Averages (AAP-current). In this paper, we modify AAP-current to follow predictions of superexperts and prove the theoretical bounds on its performance. We also consider a different protocol, where predictions and outcomes come from one-dimensional spaces $\Gamma$ and $\Omega$. AAP-current works with cumulative average losses, whereas AAP-max and AAP-incremental work with cumulative losses. However, AAP-max and AAP-incremental have a drawback that they require the knowledge of the maximum delay and the next step's delay size respectively, which makes them not applicable in real scenarios. Though AAP-current works with cumulative average losses, it does not need to know the size of the delay in advance. In the experimental part, we show that the algorithm also works well in terms the cumulative losses.

In this paper, we consider anomalies to be binary outcomes from the outcome space $\Omega = \{0, 1\}$, and predictions of experts and of the learner are anomaly probabilities from the prediction space $\Gamma = [0, 1]$. We consider two loss functions: {\em logarithmic loss} (log-loss):
\begin{equation*}
	\lambda(y, \gamma) = \begin{cases} 
		-\ln \gamma \hspace{1.865cm}\text{if $y = 1$,} \\
		-\ln (1-\gamma) \hspace{1.07cm}\text{if $y = 0$,}
	\end{cases}
\end{equation*}
and {\em square-loss}:
\begin{equation*}
	\lambda(y, \gamma) = (y - \gamma)^2.
\end{equation*}
We refer to the games with the respective loss functions as the {\em log-loss} and {\em square-loss games}.

\section{Aggregating Algorithm for Pack Averages} \label{sec:AA}
In this section, we describe AAP-current. To solve the problem of predicting as well as the best expert in the pool the Aggregating Algorithm (AA) was proposed by V.~Vovk in \cite{vovk_aggr}. The idea of the AA is to update experts' weights at each step according to their current losses. AAP-current, its modification for the case of the prediction of vector-valued outcomes, was proposed in \cite{Adamskiy2019Packs}. At step $t$ when nature reveals the outcomes $y_{t,d},~d=1,\dots, D_t$, AAP-current updates the experts' weights as follows:
\begin{equation}\label{eq:weights_update_aa}
	w_t(i) = e^{-\frac{\eta}{D_t} \sum_{d=1}^{D_t} \lambda(y_{t,d}, \xi_{t,d}(i))} w_{t-1}(i),
\end{equation}
where $\eta$ is the {\em learning rate} and $w_0(i)$ is the experts' prior distribution.

At each step $t$ and delay $d$, AAP-current outputs the prediction $\gamma_{t,d}$, which coincides with the predictions of the AA:
\begin{equation}\label{eq:gammaAA}
	\forall y \in \Omega: \lambda(y, \gamma_{t,d}) \le C g_{t,d}(y),
\end{equation}
where  $C>0$ is a positive constant and $g_{t,d}(y)$ is the {\em generalised prediction}:
\begin{equation}\label{eq:generalised}
	g_{t,d}(y) = -\frac{1}{\eta} \ln \sum_{i=1}^N e^{-\eta \lambda(y, \xi_{t,d}(i))} w_{t-1}^{*}(i),
\end{equation}
where $w_{t-1}^{*}(i)$ are normalized weights:
\begin{equation}\label{eq:weights_norm}
	w_{t-1}^{*}(i) = \frac{w_{t-1}(i)}{\sum_{i=1}^N w_{t-1}(i)}.
\end{equation}

A constant
$C>0$ is called {\em admissible} for a learning rate $\eta>0$ if there exists $\gamma_{t,d}$ which satisfies (\ref{eq:gammaAA}) for
every set of predictions
$\xi_{t,d}(i) \in \Gamma$, and every distribution
$w_{t-1}^*(i)$ (such that $\sum_{i=1}^N w_{t-1}^*(i) = 1$), and all outcomes $y\in\Omega$. The infimum of all admissible $C$ is called the {\em mixability constant}. The important class of games where the mixability constant $C=1$ is called {\em mixable}. 

Both the log-loss and the square-loss games which we consider in this paper are mixable. The AA's predictions $\gamma_{t,d}$  for the log-loss game which satisfy (\ref{eq:gammaAA}) are as follows (Section 2.2 in \cite{vovk_cols}):
\begin{equation} \label{eq:gamma_log}
	\gamma_{t,d} = \sum_{i=1}^N \xi_{t,d}(i) w^*_{t-1}(i),
\end{equation}
and for the square-loss game (Section 2.4 in \cite{vovk_cols}):
\begin{equation} \label{eq:gamma_square}
	\gamma_{t,d} = \frac{1}{2} - \frac{g_{t,d}(1) - g_{t,d}(0)}{2}.
\end{equation}

An important property of AAP-current is its theoretical guarantee. The following lemma provides the theoretical bound on the cumulative loss of AAP-current.
\begin{lemma} [Theorem 3 in \cite{Adamskiy2019Packs}] \label{lemma:AAP_bound}
	If $C$ is admissible for $\mathcal{G}$ with the learning rate $\eta$, then the learner following AAP-current suffers loss satisfying 
	\begin{equation}  \label{eq:AAP_bound}
		L^{\text{average}}_T \le C L^{\text{average}}_T(\mathcal{E}_i) + \frac{C}{\eta} \ln \frac{1}{w_0(i)}.
	\end{equation} 
\end{lemma}
The lemma shows that the regret of AAP-current is constant and does not depend on the number of steps. In the case of mixable games, AAP-current's performance is almost the same as the best expert's performance in terms of the cumulative average losses at any time step.

The optimal choice of $\eta$ depends on the type of the game and should minimise the regret from (\ref{eq:AAP_bound}). For the log-loss game the optimal $\eta = 1$ (Section 2.2 in \cite{vovk_cols}), and for the square-loss game $\eta = 2$ (Lemma 2.5 in \cite{Zhdanov2011PhD}).

\section{Competing with superexperts under delayed feedback} \label{sec:superexperts}
As was shown in the previous section, AAP-current can perform as well as the best expert in terms of the cumulative average losses. However, it is common that the nature of data changes with time, and as a consequence, the best expert also changes. It is important that a prediction strategy adapts to a constantly changing environment. Assume that there are $k+1$ segments that have different best performing `base' experts. The term {\em superexpert} \cite{kalnishkan2018ml_lectures} is referred to the strategy, which can switch between `base' experts. For example, on the first segment, the superexpert will follow the predictions of expert $\mathcal{E}_i$ and on the second segment, it will switch to expert $\mathcal{E}_j$. The goal of the learner is to compete with any such superexpert.

The Fixed-share and the Variable-share algorithms are adaptations of the AA, which allow a learner to compete with superexperts. Though the number of superexperts grows exponentially with the number of switches, it is possible to maintain only the weights of the base experts (\ref{eq:weights_norm}) instead of the superexperts' weights. The main difference of these algorithms and the AA is in the update of the experts' weights. In this section, we modify the Fixed-share and the Variable-share algorithm for the Protocol \ref{protocol_delayed} of the delayed feedback. After the weights' update (\ref{eq:weights_update_aa}) Fixed-share calculates the share updates $\tilde{w}_t(i)$:
\begin{equation}\label{eq:fs_update}
	\tilde{w}_t(i) = (1-\alpha) w_t(i) + \frac{\alpha}{N - 1} \sum_{j \ne i} w_t(j)
\end{equation}
and the Variable-share update is as follows:
\begin{equation}\label{eq:vs_update}
	\tilde{w}_t(i) = (1-\alpha)^{\frac{1}{D_t} \sum_{d=1}^{D_t}\lambda(y_{t,d}, \xi_{t, d}(i))} w_t(i) \\
	+ \sum_{j \ne i} \frac{\left(1 - (1 - \alpha)^{\frac{1}{D_t} \sum_{d=1}^{D_t}\lambda(y_{t,d}, \xi_{t, d}(j))} \right)}{N - 1} w_t(j),
\end{equation}
where $\alpha$ is the probability of switching between experts. 
The generalised predictions (\ref{eq:generalised}) are then calculated with the normalised weights $\tilde{w}^*_t(i)$ instead of $w^*_t(i)$. Note that if $\alpha = 0$ the algorithms coincide with the AAP-current. In Fixed-share each expert shares a fraction $\alpha$ of its weight with other experts. However, if one expert has the best performance for a long period of time, it is not optimal to keep sharing its weight with others. To overcome this problem, in Variable-share experts share their weights only if they have a large loss, but do not share if they perform well. However, if an expert starts to predict badly its weights will be shared with other experts, and the next best expert's weights will recover quickly. Variable-share, however, works only for the games with bounded losses. For this reason, we apply Variable-share for the square-loss game and apply Fixed-share for the log-loss game.

The following is the pseudo-code of the algorithms (see optimal values of $\eta$ for the log-loss and the square-loss games at the end of Section \ref{sec:AA}):
\begin{protocol}[Fixed-share and Variable-share under delayed feedback]~
	{\tt
		\label{aap_incremental}
		\begin{tabbing}
			\quad\=\quad\=\quad\=\quad=\quad\kill
			Input: mixability constant $C$, learning rate $\eta$, switching rate $\alpha$ \\
			1 Initialise experts' weights $\tilde{w}_0 (i)= 1 / N$, $i = 1, \dots, N$ \\
			2 FOR $t=1,2,\ldots$\\
			3 \>\> normalise weights $\tilde{w}^*_{t-1} (i) = \tilde{w}_{t-1} (i) / \sum_{i=1}^N \tilde{w}_{t-1} (i)$ \\
			4 \>\> FOR $d=1,\dots, D_t$ \\
			5 \>\>\> read experts' predictions $\xi_{t, d}(i)$, $i=1,\ldots,N$\\
			6 \>\>\> learner predicts $\gamma_{t, d} \in \Gamma$, such that \\
			\>\>\> $\forall y \in \Omega: \lambda(y, \gamma_{t, d}) \le -\frac{C}{\eta} \ln \sum_{i=1}^N e^{-\eta \lambda(y, \xi_{t, d}(i))} \tilde{w}_{t-1}^{*}(i)$ \\
			7 \>\> END FOR \\
			8 \>\> observe the outcomes $y_{t, d} \in \Omega, ~ d = 1, \dots, D_t$\\
			9 \>\> update the intermediate experts' weights \\
			\>\>\> $w_t(i) = \tilde{w}_{t-1}(i) e^{-\frac{\eta}{D_t} \sum_{d=1}^{D_t} \lambda(y_{t, d}, \xi_{t, d}(i))},~i = 1,\dots,N$ \\
			10 \>\> if Fixed-share: \\
			11 \>\>\> pool = $\sum_{i=1}^N \alpha w_t(i)$\\
			12 \>\>\> update the shares: \\
			\>\>\> $\tilde{w}_t(i) = (1-\alpha) w_t(i) + \frac{1}{N - 1}(\text{pool} - \alpha w_t(i))$ \\
			13 \>\> if Variable-share: \\
			14 \>\>\> pool = $\sum_{i=1}^N \left(1 - (1 - \alpha)^{\frac{1}{D_t} \sum_{d=1}^{D_t} \lambda(y_{t,d}, \xi_{t,d}(i))}\right) w_t(i)$ \\
			15 \>\>\> update the shares: \\
			\>\>\> $\tilde{w}_t(i) = (1-\alpha)^{\frac{1}{D_t} \sum_{d=1}^{D_t}\lambda(y_{t,d}, \xi_{t, d}(i))} w_t(i) + \frac{1}{N - 1} \left(\text{pool} - \left(1 - (1 - \alpha)^{\frac{1}{D_t} \sum_{d=1}^{D_t}\lambda(y_{t,d}, \xi_{t, d}(i))}\right) w_t(i) \right)$\\
			16 END FOR
		\end{tabbing}
	}
\end{protocol}
Note that for the log-loss and the square-loss games the learner outputs predictions according to (\ref{eq:gamma_log}) and (\ref{eq:gamma_square}) respectively with normalised weights $\tilde{w}^*_t(i)$.

We denote $\mathcal{S}_{T, N, k, \bf{t}, \bf{e}}$ to be the superexpert with $T$ number of steps, $N$ base experts and $k$ switches ($k < T$). The tuple $\bf{t}$ divides the data into $k+1$ segments $[t_0, t_1), [t_1, t_2),\cdots,[t_k, t_{k+1})$. The tuple $\bf{e}$ has $k$ elements $(e_0, e_1, \cdots, e_k)$, such that $1 \le e_j \le N$ and $e_j \ne e_{j+1}$. The element $e_j$ denotes the expert $\mathcal{E}_{e_j}$ associated with the $j$th segment $[t_j, t_{j+1})$.

The following lemma provides the upper bound on the cumulative average loss of the Fixed-share algorithm.
\begin{theorem} \label{theorem:bound_fs}
	For a mixable game $\mathcal{G}$, a learning rate $\eta$ and for any superexpert $\mathcal{S}_{T, N, k, \bf{t}, \bf{e}}$ the total average loss of the Fixed-share algorithm with parameter $\alpha$ satisfies
	\begin{equation} \label{eq:bound_fs}
		L_T^{\textrm{average}} \le CL_T^{\textrm{average}} (\mathcal{S}_{T, N, k, \bf{t}, \bf{e}}) + \frac{C}{\eta} \left( \ln N + k \ln \frac{N-1}{\alpha}
		+ (T - 1 - k) \ln \frac{1}{1 - \alpha} \right).
	\end{equation}
\end{theorem}
The minimum of the bound (\ref{eq:bound_fs}) is achieved by setting $\alpha = \frac{k}{T-1}$. However, the number of steps $T$ and the number of switches $k$ is usually not known in advance. In practice, it is often recommended to find the optimal $\alpha$ experimentally \cite{warmuth_tracking_original}. In the case when the upper bound on the number of steps $\hat{T}$ and the lower bound on the number of switches $\hat{k}$ are known in advance the following corollary is realised.
\begin{corollary} \label{cor:bound_fs}
	Under the conditions of Theorem \ref{theorem:bound_fs} and any positive reals $\hat{T}$ and $\hat{k}$, such that $\hat{k} < \hat{T} - 1$ by setting $\alpha = \hat{k} / (\hat{T} - 1)$, the total average loss of the Fixed-share algorithm can be bounded by
	\begin{equation} \label{eq:cor_fs}
		L_T^{\textrm{average}} \le C L_T^{\textrm{average}} (\mathcal{S}_{T, N, k, \bf{t}, \bf{e}}) + \frac{C}{\eta} \left( \ln N + k \Big(\ln \frac{\hat{T} - 1}{\hat{k}} + \ln(N-1)\Big) + \hat{k}   \right),
	\end{equation}
	where $\mathcal{S}_{T, N, k, \bf{t}, \bf{e}}$ is any superexpert such that $T \le \hat{T}$ and $k \ge \hat{k}$. 
\end{corollary}

The following lemma provides the upper bound on the loss of the Variable-share algorithm.
\begin{theorem} \label{theorem:bound_vs}
	Let $\mathcal{G}$ be a mixable game with a learning rate $\eta$, and a loss at each step bounded by an interval $[0, 1]$. Then for any superexpert $\mathcal{S}_{T, N, k, \bf{t}, \bf{e}}$ the total average loss of the Variable-share algorithm with parameter $\alpha$ satisfies
	\begin{equation} \label{eq:bound_vs}
		L_T^{\textrm{average}} \le C\Big(1 + \frac{1}{\eta} \ln \frac{1}{1-\alpha} \Big) L_T^{\textrm{average}} (\mathcal{S}_{T, N, k, \bf{t}, \bf{e}}) +\frac{C}{\eta} \left(\ln N + k \Big( \eta + \ln \frac{N-1}{\alpha (1 - \alpha)} \Big) \right).
	\end{equation}
\end{theorem}

The coefficient in front of the loss of the superexpert in (\ref{eq:bound_vs}) is greater than one, meaning that asymptotically the loss of Variable-share is worse than the loss of the superexpert $\mathcal{S}_{T, N, k, \bf{t}, \bf{e}}$. The following corollary improves the asymptotic behaviour of the upper bound on Variable-share.
\begin{corollary} \label{cor:bound_vs}
	Under the conditions of Theorem \ref{theorem:bound_fs} and any positive reals $\hat{L}$ and $\hat{k}$, by setting $\alpha = \frac{\hat{k}}{2\hat{k} + \hat{L}}$, the loss of the Variable-share algorithm can be bounded as
	\begin{equation} \label{eq:cor_vs1}
		L_T^{\textrm{average}} \le C L_T^{\textrm{average}} (\mathcal{S}_{T, N, k, \bf{t}, \bf{e}}) 
		+ \frac{C}{\eta} \left( \ln N + k \Big(\ln \Big( \frac{\hat{L}}{\hat{k}} \Big) + \ln(N-1) + \ln \frac{9}{2} + \eta \Big) + \hat{k} \right),
	\end{equation}
	where $\mathcal{S}_{T, N, k, \bf{t}, \bf{e}}$ is any superexpert such as $L^{\textrm{average}}_T (\mathcal{S}_{T, N, k, \bf{t}, \bf{e}}) \le \hat{L}$, and in addition $\hat{k} \le \hat{L}$.  
	
	For any superexpert $\mathcal{S}_{T, N, k, \bf{t}, \bf{e}}$ such that $L^{\textrm{average}}_T (\mathcal{S}_{T, N, k, \bf{t}, \bf{e}}) \le \hat{L}$ and $\hat{k} \ge \hat{L}$, the upper bound is
	\begin{equation} \label{eq:cor_vs2}
		L_T^{\textrm{average}} \le C L_T^{\textrm{average}} (\mathcal{S}_{T, N, k, \bf{t}, \bf{e}}) + \frac{C}{\eta} \left( \ln N + k \Big(\ln (N -1) + \ln \frac{9}{2} + \eta \Big) + \frac{1}{2}\hat{k} \right).
	\end{equation}
\end{corollary}

Lemma \ref{lemma:AAP_bound} shows that the regret between AAP-current and the best expert is $O(1)$. It is possible to compete with the best superexpert $\mathcal{S}_{T, N, k, \bf{t}, \bf{e}}$ which can switch between experts. Lemma \ref{theorem:bound_fs} and Corollary \ref{cor:bound_fs} provide the upper bounds for cumulative average losses of Fixed-share. The regret between Fixed-share and the best superexpert $\mathcal{S}_{T, N, k, \bf{t}, \bf{e}}$ with the number of base experts $N$, the number of switches $k \ge \hat{k}$ and the number of steps $T \le \hat{T}$ is of the order $O(k\ln N + k \ln (\hat{T}/\hat{k}))$. The Variable-share algorithm has a more sophisticated weights' update and a better bound which does not depend on the number of steps $T$. However, it can only be applied to games with bounded losses. Lemma \ref{theorem:bound_vs} and Corollary \ref{cor:bound_vs} provide the upper bounds for cumulative average losses of Variable-share. The regret between Variable-share and the best superexpert is of the order $O(k \ln N + k \ln (\hat{L} / k))$, for any superexpert $\mathcal{S}_{T, N, k, \bf{t}, \bf{e}}$ such as $L^{\textrm{average}}_T(\mathcal{S}_{T, N, k, \bf{t}, \bf{e}}) \le \hat{L}$. We prove the worst-case upper bounds on the performance of Fixed-share and Variable-share under delayed feedback scenario. These bounds are similar to those in \cite{warmuth_tracking_original}, but of a different kind as they provide guarantees on cumulative average losses instead of cumulative losses. The proofs of Theorems \ref{theorem:bound_fs} and \ref{theorem:bound_vs} are similar to those without delays in \cite{warmuth_tracking_original} and can be found in Appendix.

\section{Experiments} \label{sec:results}
In this section, we perform experiments on the NAB benchmark and compare the performance of AAP-current, Fixed-share and Variable-share with other algorithms. All the results of this paper are available openly\textsuperscript{\ref{anomaly}}. The NAB benchmark has a scoreboard, and anyone can test their algorithms and take part in the competition. The current number of models is 15.  We run the algorithms with experts, which are the algorithms available on the NAB benchmark, for different parameters of $\alpha$ starting from 0 to 0.3. We also artificially insert different lengths of delays in the data to test our implementations. For log-loss, we apply the Fixed-share and for square-loss, we apply the Variable-share algorithms. Note, that for $\alpha = 0$ the predictions coincide with the AAP's predictions. The NAB data corpus includes two artificial and five real data types each of which includes several time-series. Examples of the real time-series include AWS server metrics; temperature sensor data of an industrial machine: a shut-down, catastrophic failure; the key hold timings for several users of a computer; online advertisement clicking rates: cost-per-click and cost per thousand impressions\textsuperscript{\ref{nab}}. One of the artificial datasets does not include anomalies. The total number of time-series is 58. The data contains both contextual and collective anomalies and the nature of data changes over time. The anomalies in these datasets do not come as single points, instead, the datasets contain anomaly `windows'. It helps to detect anomalies in advance and take appropriate preventive measures.

\subsection{Analysis of losses and classification metrics}
First, we compare algorithms in terms of losses and classification metrics. Table \ref{table:auc} shows the Area under Curve (AUC) of different algorithms calculated on all time-series data available at NAB. The resulting table includes AAP-current (AAP), Fixed-share, and Variable-share for different parameters $\alpha$. The columns show AUC for different delays, which vary from one to 100. Delay one means that the feedback comes immediately after we predict the current time step, whereas when the delay is 100, we receive feedback after we observe and make predictions for 100 points. We also include the random delay, which at each time step assigns a random delay from 20 to 100. For the comparison, we include the top three competitor algorithms with the highest values of AUC. We can see from the table that both Fixed-share and Variable-share have significantly higher AUC compare to randomCutForest, which has the highest AUC on NAB. Only AAP with square-loss and delay 100 has the AUC lower than randomCutForest. We observe that lower delay results in higher AUC, which is expected. The more frequent feedback leads to the more frequent weights update, and therefore better prediction accuracy. We can see that for large delays the performance of AAP degrades significantly, whereas Fixed-share and Variable-share perform much better. We also observe that the larger switching rates $\alpha$ result in slightly better performance. We have similar results for F-score at Table \ref{table:f_score}. The F-scores in the table are the maximum F-scores for each algorithm. The performance in terms of F-score of AAP, Fixed-share, and Variable-share for all delays and switching rates are significantly better compared to algorithms from NAB.

Table \ref{table:log_loss} illustrates the total log-loss of different algorithms and the top three competitors with the lowest log-loss. We can see that only AAP with the square-loss and delays 50 and more has the losses higher than randomCutForest, which is the best competitor in terms of log-loss. Different from the classification metrics, the lowest log-losses achieved with smaller switching rates $\alpha$, such as 0.05 and 0.10. Similar results are observed for the total square-loss at Table \ref{table:square_loss}.

\begin{table}[ht]
	\begin{center}
		\begin{tabular}{|c|ccccc|}
			\hline
			algorithm \textbackslash \hspace{0.1cm} delay &      1 &      20 &      50 &      100 &      \makecell{random\\ $\big[20, 100\big]$} \\
			\hline
			AAP, log-loss     &  0.997 &  0.938 &  0.809 &  0.696 &  0.771 \\
			Fixed, $\alpha = 0.01$   &  0.998 &  0.970 &  0.922 &  0.833 &  0.894 \\
			Fixed, $\alpha = 0.05$   &  0.998 &  0.974 &  0.936 &  0.869 &  0.912 \\
			Fixed, $\alpha = 0.10$  &  0.998 &  0.975 &  0.940 &  0.881 &  0.918 \\
			Fixed, $\alpha = 0.30$   &  0.998 &  0.975 &  0.942 &  0.893 &  0.920 \\
			\hline
			AAP, square-loss  &  0.991 &  0.826 &  0.657 &  0.581 &  0.639 \\
			Variable, $\alpha = 0.01$ &  0.997 &  0.966 &  0.926 &  0.874 &  0.904 \\
			Variable, $\alpha = 0.05$  &  0.998 &  0.976 &  0.949 &  0.906 &  0.925 \\
			Variable, $\alpha = 0.10$  &  0.998 &  0.978 &  0.954 &  0.909 &  0.929 \\
			Variable, $\alpha = 0.30$  &  0.998 &  0.979 &  0.955 &  0.908 &  0.929 \\
			\hline
			randomCutForest  & && 0.616&&\\
			knncad       &&&      0.600 &&\\
			skyline         &&&  0.566 &&\\
			\hline
		\end{tabular}
	\end{center}
	\caption{AUC}
	\label{table:auc}
\end{table}

\begin{table}[ht]
	\begin{center}
		\begin{tabular}{|c|ccccc|}
			\hline
			algorithm \textbackslash \hspace{0.1cm} delay &      1 &      20 &      50 &      100 &      \makecell{random\\ $\big[20, 100\big]$} \\
			\hline
			AAP, log-loss      &  0.978 &  0.684 &  0.418 &  0.296 &  0.372 \\
			Fixed, $\alpha = 0.01$      &  0.988 &  0.832 &  0.651 &  0.455 &  0.596 \\
			Fixed, $\alpha = 0.05$     &  0.990 &  0.863 &  0.714 &  0.540 &  0.667 \\
			Fixed, $\alpha = 0.10$     &  0.990 &  0.877 &  0.744 &  0.584 &  0.701 \\
			Fixed, $\alpha = 0.30$     &  0.984 &  0.897 &  0.788 &  0.650 &  0.748 \\
			\hline
			AAP, square-loss   &  0.961 &  0.521 &  0.319 &  0.237 &  0.300 \\
			Variable, $\alpha = 0.01$   &  0.979 &  0.809 &  0.677 &  0.553 &  0.642 \\
			Variable, $\alpha = 0.05$   &  0.983 &  0.870 &  0.775 &  0.672 &  0.740 \\
			Variable, $\alpha = 0.10$  &  0.984 &  0.888 &  0.802 &  0.690 &  0.763 \\
			Variable, $\alpha = 0.30$  &  0.987 &  0.896 &  0.791 &  0.677 &  0.751 \\
			\hline
			randomCutForest   &&& 0.233 &&\\
			htmjava    &&&        0.218 &&\\
			knncad         &&&    0.216 &&\\
			\hline
		\end{tabular}
	\end{center}
	\caption{F-score}
	\label{table:f_score}
\end{table}

\begin{table}[ht]
	\begin{center}
		\begin{tabular}{|c|ccccc|}
			\hline
			algorithm \textbackslash \hspace{0.1cm} delay &      1 &      20 &      50 &      100 &      \makecell{random\\ $\big[20, 100\big]$} \\
			\hline
			AAP, log-loss    &  12.1 &  61.1 &   95.3 &  106.2 &  101.4 \\
			Fixed, $\alpha = 0.01$     &  11.8 &  36.3 &   61.5 &   87.5 &   69.9 \\
			Fixed, $\alpha = 0.05$     &  19.9 &  37.6 &   56.6 &   78.0 &   63.2 \\
			Fixed, $\alpha = 0.10$    &  28.4 &  42.6 &   58.3 &   76.4 &   63.8 \\
			Fixed, $\alpha = 0.30$    &  56.8 &  64.2 &   73.7 &   84.6 &   77.0 \\
			\hline
			AAP, square-loss  &  17.4 &  90.0 &  109.6 &  115.4 &  112.3 \\
			Variable, $\alpha = 0.01$  &  12.4 &  40.6 &   59.2 &   75.6 &   64.8 \\
			Variable, $\alpha = 0.05$  &  11.6 &  34.1 &   49.7 &   66.0 &   56.1 \\
			Variable, $\alpha = 0.10$ &  11.3 &  32.9 &   48.8 &   66.2 &   55.9 \\
			Variable, $\alpha = 0.30$ &  11.3 &  35.6 &   52.5 &   70.0 &   60.0 \\
			\hline
			randomCutForest   &&& 107.5 &&\\
			htmjava        &&&    153.9 &&\\
			numentaTM     &&&     161.4 &&\\
			\hline
		\end{tabular}
	\end{center}
	\caption{Total logarithmic loss ($\times 10^3$)}
	\label{table:log_loss}
\end{table}

\begin{table}[ht]
	\begin{center}
		\begin{tabular}{|c|ccccc|}
			\hline
			algorithm \textbackslash \hspace{0.1cm} delay &      1 &      20 &      50 &      100 &      \makecell{random\\ $\big[20, 100\big]$} \\
			\hline
			AAP, log-loss      &   2.9 &  16.5 &  26.4 &  29.0 &  27.6 \\
			Fixed, $\alpha = 0.01$      &   2.5 &   9.7 &  17.4 &  24.9 &  19.6 \\
			Fixed, $\alpha = 0.05$      &   3.0 &   8.9 &  15.1 &  21.8 &  17.1 \\
			Fixed, $\alpha = 0.10$     &   4.1 &   9.2 &  14.6 &  20.7 &  16.4 \\
			Fixed, $\alpha = 0.30$     &  10.3 &  13.4 &  17.1 &  21.2 &  18.3 \\
			\hline
			AAP, square-loss   &   3.8 &  23.5 &  28.9 &  30.0 &  29.5 \\
			Variable, $\alpha = 0.01$   &   2.9 &  10.7 &  16.4 &  21.3 &  17.9 \\
			Variable, $\alpha = 0.05$   &   2.8 &   8.8 &  13.6 &  18.6 &  15.2 \\
			Variable, $\alpha = 0.10$  &   2.7 &   8.6 &  13.6 &  18.8 &  15.3 \\
			Variable, $\alpha = 0.30$  &   2.7 &  10.2 &  15.3 &  19.9 &  16.9 \\
			\hline
			randomCutForest    &&& 29.4 &&\\
			htmjava            &&& 31.5 &&\\
			numentaTM       &&&   31.8 &&\\
			\bottomrule
		\end{tabular}
	\end{center}
	\caption{Total square loss ($\times 10^3$)}
	\label{table:square_loss}
\end{table}

\subsection{Visualisation of predictions}
In this section, we visualise the predictions of our proposed approaches for different $\alpha$ and delays. Figure \ref{fig:predictions_fixed10d20} shows the predictions of Fixed-share with $\alpha = 0.1$ and delay 20 for three time-series from the real dataset with known anomaly causes. The blue lines correspond to the algorithm's predictions, whereas the orange lines represent the anomaly windows. Figure \ref{fig:predictions_variable10d100} illustrates the predictions of Variable-share with $\alpha = 0.1$ and delay 100. The optimal thresholds for these algorithms that maximise the F-score are 0.38 and 0.24 for Fixed-share and Variable-share respectively. It means that anything which is above that threshold should be classified as an anomaly. 
It is consistent with the figures, and it is clear that the algorithms `catch' the shapes of the anomaly windows. Though Fixed-share with lower delay produces more confident predictions as it assigns higher probabilities within the anomaly windows. For comparison figure \ref{fig:predictions_randomcutforest} illustrates the predictions of randomCutForest, which was the best NAB competitor in terms of classification metrics and loss functions. We can see that randomCutForest's predictions are more `spiky'. The optimal threshold of randomCutForest on the whole dataset is around 0.12. 

\begin{figure}[ht]
	\begin{center}
		\includegraphics[scale=0.25]{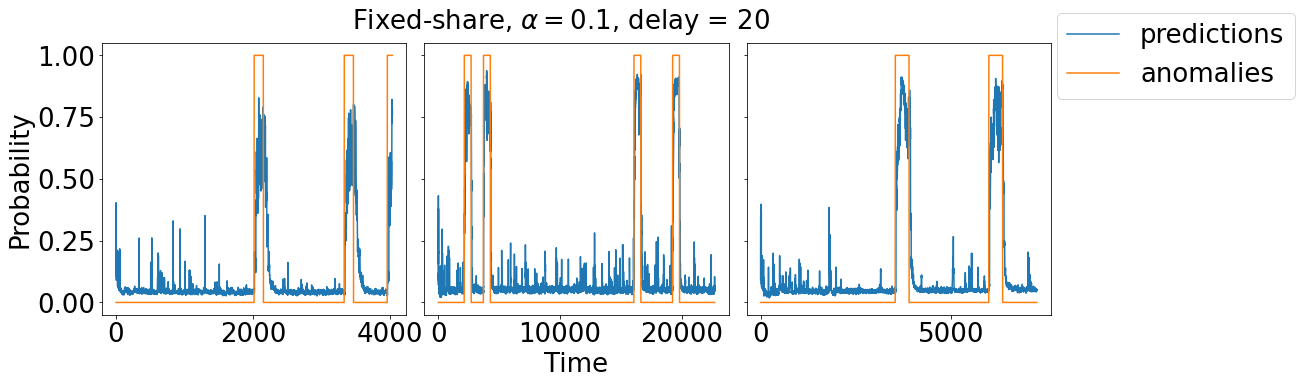}
		\caption{Predictions of Fixed-share with $\alpha = 0.1$}
		\label{fig:predictions_fixed10d20}
	\end{center}
\end{figure}

\begin{figure}[ht]
	\begin{center}
		\includegraphics[scale=0.25]{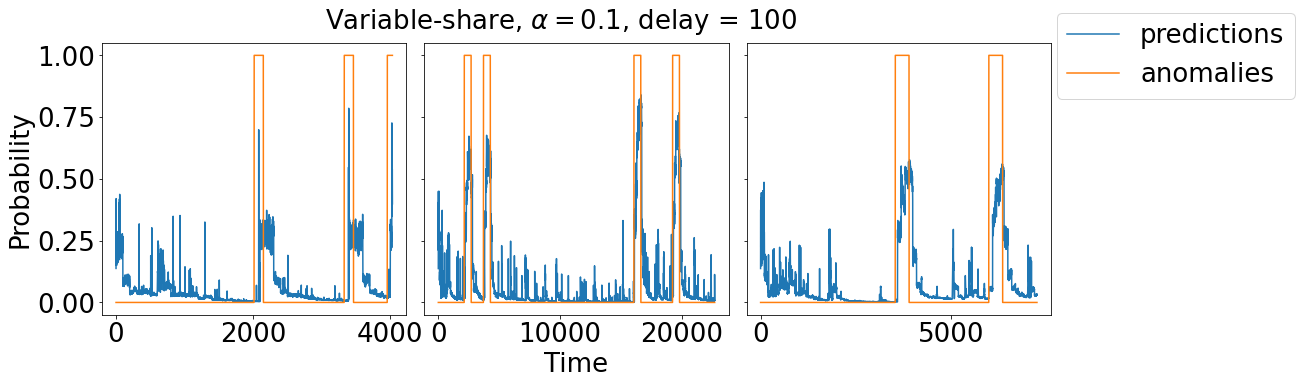}
		\caption{Predictions of Variable-share with $\alpha = 0.5$}
		\label{fig:predictions_variable10d100}
	\end{center}
\end{figure}

\begin{figure}[ht]
	\begin{center}
		\includegraphics[scale=0.25]{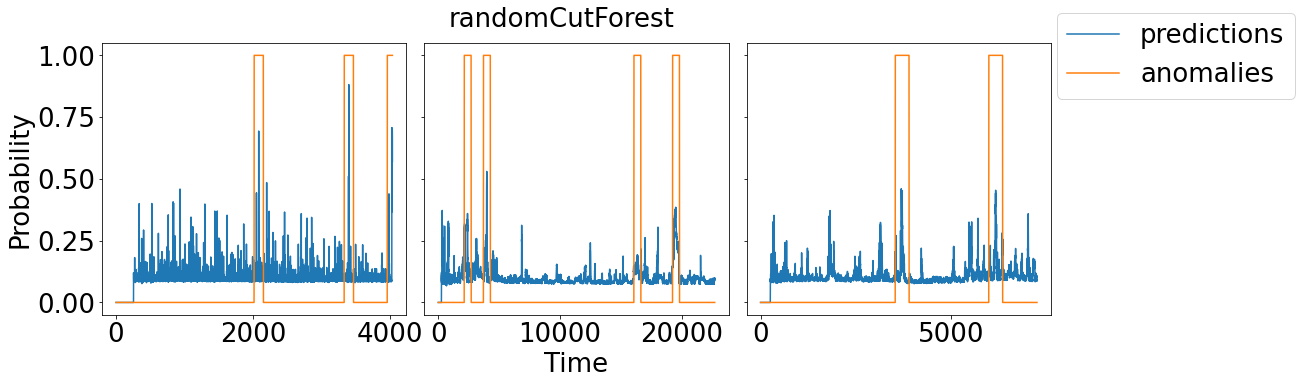}
		\caption{Predictions of randomCutForest}
		\label{fig:predictions_randomcutforest}
	\end{center}
\end{figure}

\subsection{Weights analysis} 
In this section, we analyse the weight updates of the base experts for AAP-current (AAP), Fixed-share, and Variable-share with delay 100. Figure \ref{fig:weights_update} illustrates the normalised weight updates of the algorithms for the real dataset with known causes of the machine system failure, which contains time-series of the machine's temperatures. We can observe from the graphs that AAP with log-loss quickly identifies the best expert and follows the predictions of randomCutForest. AAP with square-loss follows the predictions of randomCutForest and switches to skyline at the end. In the case of AAP, the best expert usually acquires the largest weight close to one, whereas other experts almost do not participate in making predictions. Fixed-share observes less drastic weight changes compared to AAP. For $\alpha = 0.5$. we observe that several expert models acquire large weights. However, other models have non-zero weights and therefore participate in making predictions. The impact of $\alpha$ becomes more obvious when we increase it even more. For Fixed-share with a switching rate $\alpha = 0.3$  the base experts share the greatest amount of their weights and therefore participate more in making predictions. We observe a different behaviour in the case of Variable-share. If one expert performs well for a long period of time, it stops sharing its weight with others. However, if an expert starts to suffer a large loss, its weights will be shared with other experts, which will help the recovery of the weights of the next best expert. We can see from the graphs that Variable-share allows rapid switching between base experts, reflecting their current performance. 

\begin{figure}[ht] 
	\begin{subfigure}[b]{0.5\linewidth}
		\centering
		\includegraphics[width=0.95\linewidth]{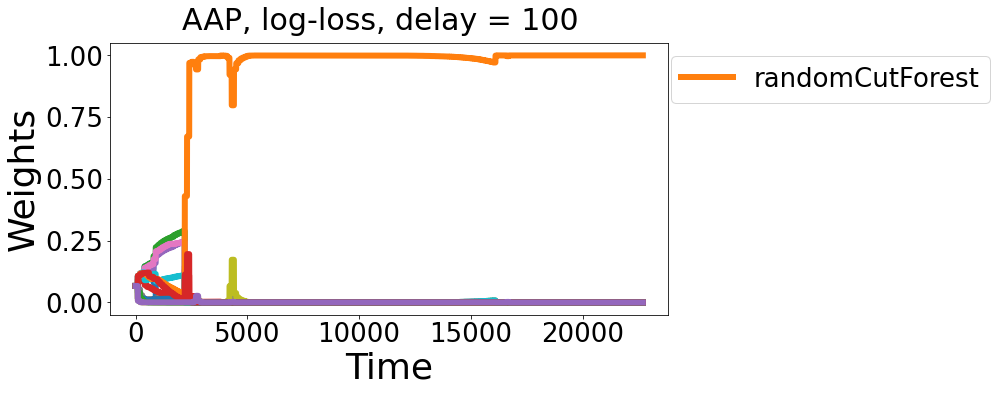} 
		\vspace{4ex}
	\end{subfigure}
	\begin{subfigure}[b]{0.5\linewidth}
		\centering
		\includegraphics[width=0.95\linewidth]{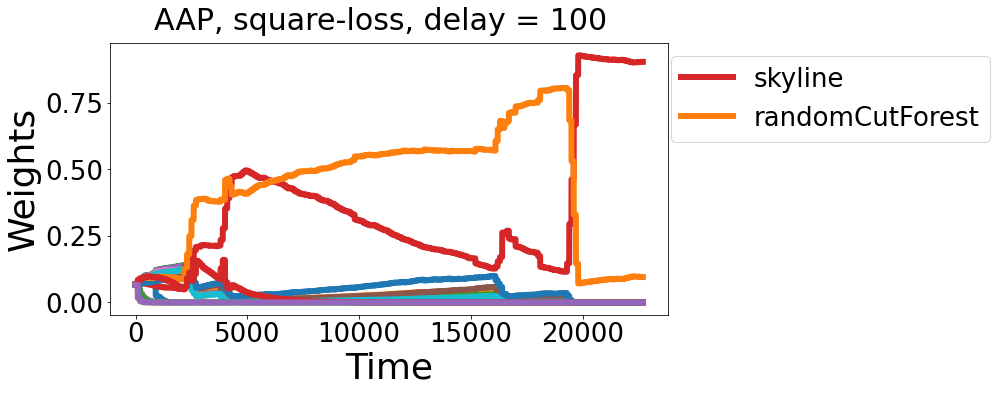} 
		\vspace{4ex}
	\end{subfigure} 
	\begin{subfigure}[b]{0.5\linewidth}
		\centering
		\includegraphics[width=0.95\linewidth]{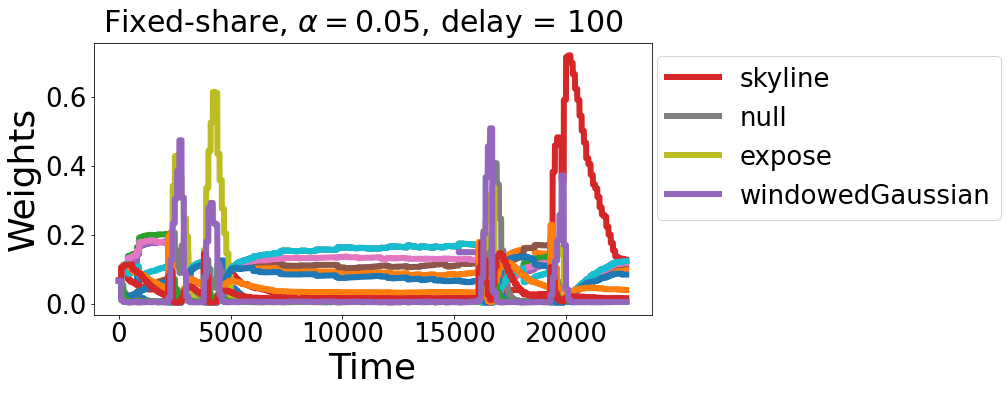} 
		\vspace{4ex}
	\end{subfigure}
	\begin{subfigure}[b]{0.5\linewidth}
		\centering
		\includegraphics[width=0.95\linewidth]{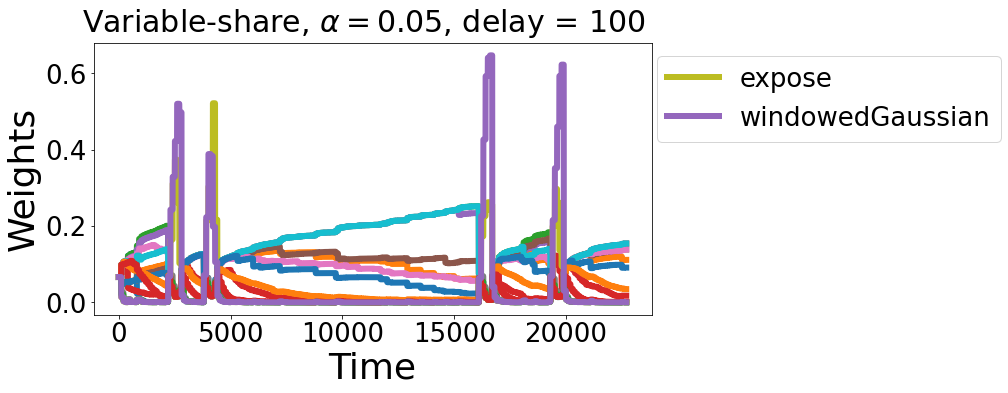} 
		\vspace{4ex}
	\end{subfigure} 
	\begin{subfigure}[b]{0.5\linewidth}
		\centering
		\includegraphics[width=0.95\linewidth]{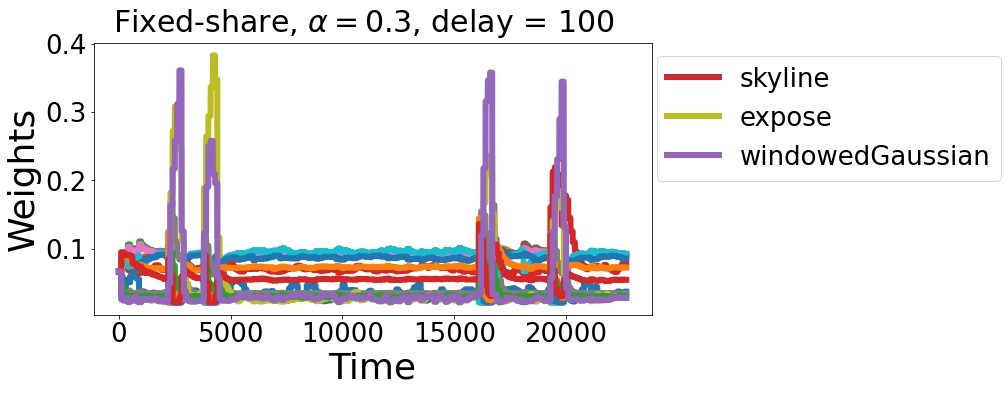} 
	\end{subfigure}
	\begin{subfigure}[b]{0.5\linewidth}
		\centering
		\includegraphics[width=0.95\linewidth]{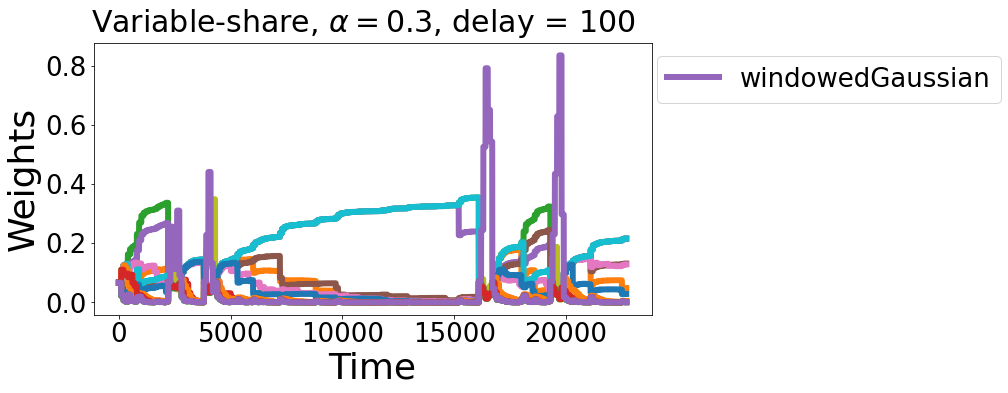} 
	\end{subfigure} 
	\caption{Weights update of AAP, Fixed-share and Variable-share for different $\alpha$}
	\label{fig:weights_update} 
\end{figure}

\subsection{Theoretical bounds}
An important property of the considered algorithms is their theoretical bounds. At any point in the future, we are sure that the total average losses of the algorithms will be close to the best expert's losses in the case of AAP-current, and close to the superexpert's losses in the case of Fixed-share and Variable-share. In the following experiments, we visualise the theoretical bound of the AAP. For this purpose, we picked one time-series of the perfect square wave from the artificial dataset without anomalies. Recall, that AAP-current coincides with AA when delay is equal to one. Figure \ref{fig:loss_diff_fixed0d1} shows the difference between the total logarithmic losses of the expert models and the AA. Therefore, if the loss difference is below zero, the AA performs worse than the expert. For illustration, we pick three models: numenta, earthgeckoSkyline and bayesChangePt. The last two models are taken because they have the lowest total losses on the artificial dataset. The figure also illustrates the theoretical bound of the AA. The meaning of this bound is that for any point in the future, even if the AA performs worse than the best expert, it will not perform much worse. In the case of AA, the theoretical bound is a constant and does not depend on time. From the figure, we can see that the bound is `tight'.  Figure \ref{fig:loss_diff_variable0d1} illustrates the same experiment for square loss. Figure \ref{fig:loss_diff_d50} illustrates the theoretical bound on the cumulative average losses for AAP with delay 50. The loss difference shown in the figure is between the cumulative average losses of the competitors and AAP. Similarly, we have a constant regret in the case of AAP, though the bound is not as `tight' as in the case of `no delay'.

\begin{figure}
	\centering
	\begin{subfigure}{.5\textwidth}
		\centering
		\includegraphics[width=1\linewidth]{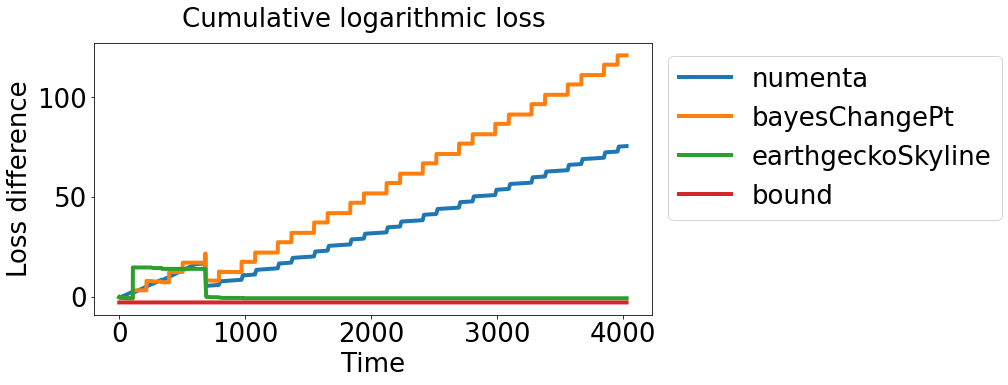}
		\caption{Log-loss}
		\label{fig:loss_diff_fixed0d1}
	\end{subfigure}%
	\begin{subfigure}{.5\textwidth}
		\centering
		\includegraphics[width=1\linewidth]{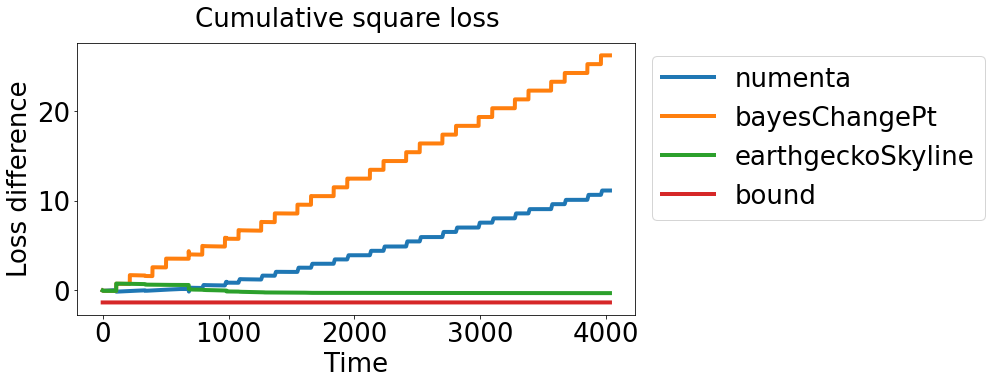}
		\caption{Square-loss}
		\label{fig:loss_diff_variable0d1}
	\end{subfigure}
	\caption{Loss difference between cumulative losses of algorithms and Aggregating Algorithm without delays}
	\label{fig:loss_diff_d1}
\end{figure}

\begin{figure}
	\centering
	\begin{subfigure}{.5\textwidth}
		\centering
		\includegraphics[width=1\linewidth]{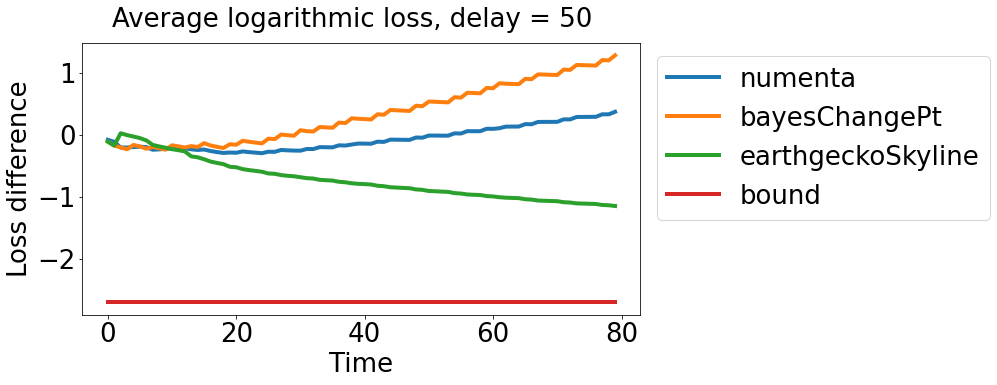}
		\caption{Log-loss}
		\label{fig:loss_diff_fixed0d50}
	\end{subfigure}%
	\begin{subfigure}{.5\textwidth}
		\centering
		\includegraphics[width=1\linewidth]{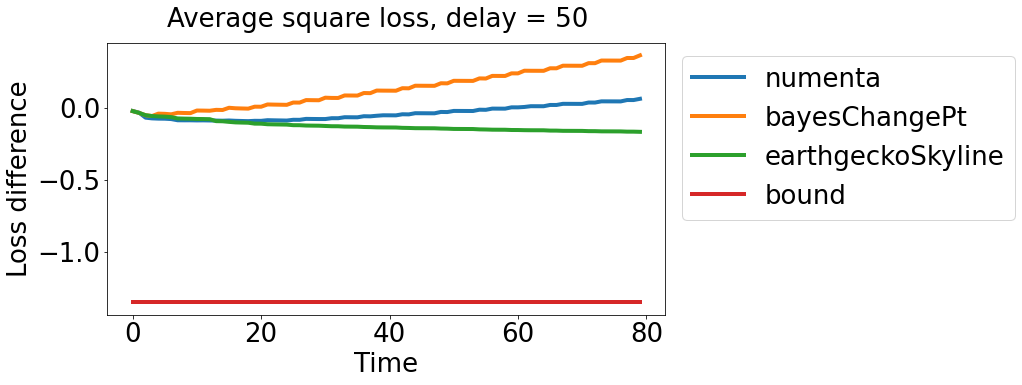}
		\caption{Square-loss}
		\label{fig:loss_diff_variable0d50}
	\end{subfigure}
	\caption{Loss difference between cumulative average losses of algorithms and AAP-current with delay 50}
	\label{fig:loss_diff_d50}
\end{figure}

\section{Conclusions}
In this work, we have proposed an adaptation of the Fixed-share and Variable-share algorithms to delayed feedback. We proved their worst-case upper bounds on the cumulative average losses. We then applied these approaches to the problem of anomaly detection that results in a new method for aggregating unsupervised anomaly detection algorithms and incorporating feedback that might come with delay. Apart from the theoretical guarantees, the approach is highly explainable as the experts' weights reflect their performances.

We have tested the proposed approaches on the NAB benchmark, which provides an open-source environment for testing anomaly detection algorithms on streaming datasets. Experimental results show that the proposed approaches, which combine models, provide significantly better results than any single model available at the NAB repository in terms of several losses and classification metrics. We explore the predictions of the proposed methods for different values of the switching rate $\alpha$. We explain the behaviour of the algorithms by analysing the weights' updates for different $\alpha$. At last, we illustrate the theoretical bounds of the proposed approaches.

\section*{Acknowledgments}
This work was supported, in whole or in part, by the Bill \& Melinda Gates Foundation [INV-001309]. Under the grant conditions of the Foundation, a Creative Commons Attribution 4.0 Generic License has already been assigned to the Author Accepted Manuscript version that might arise from this submission.

\bibliographystyle{unsrt}
\bibliography{mylib}
\nocite{Hunter:2007}

\section*{Appendix}

\begin{lemma} \label{lemma:weights_fs1}
	For any outcomes and experts' predictions the following inequality on experts' weights holds for any $t \le t^\prime$
	\begin{equation}\label{eq:weights_fs1}
		\frac{w_{t^\prime}(i)}{\tilde{w}_{t-1}(i)} \ge e^{-\eta  \sum_{r=t}^{t^\prime} \frac{1}{D_r}\sum_{d=1}^{D_r} \lambda(y_{r, d}, \xi_{r, d}(i))} (1-\alpha)^{(t^\prime-t)}.
	\end{equation}
	\begin{proof}
		The proof is similar to the proof of Lemma 2 in \cite{warmuth_tracking_original}. By combining (\ref{eq:fs_update}) and the intermediate weights update (line 9 of Protocol \ref{aap_incremental}) the share weights update on step $t$ is
		\begin{equation*}
			\tilde{w}_t(i) = \frac{\alpha}{N-1} \sum_{j \ne i} w_t(j) + (1-\alpha) e^{-\frac{\eta}{D_t}  \sum_{d=1}^{D_t} \lambda(y_{t, d}, \xi_{t, d}(i))} \tilde{w}_{t-1}(i).
		\end{equation*}
		By dropping the first term on the right-hand side we get the following inequality
		\begin{equation*}
			\tilde{w}_t(i) \ge (1-\alpha) e^{-\frac{\eta}{D_t}  \sum_{d=1}^{D_t} \lambda(y_{t, d}, \xi_{t, d}(i))} \tilde{w}_{t-1}(i).
		\end{equation*}
		The intermediate weights update on trial $t^\prime$ is
		\begin{multline*}
			w_{t^\prime}(i) = \tilde{w}_{t^\prime-1}(i) e^{-\frac{\eta}{D_{t^\prime}}  \sum_{d=1}^{D_{t^\prime}} \lambda(y_{t^\prime, d}, \xi_{t^\prime, d}(i))}\\
			\ge  e^{-\frac{\eta}{D_{t^\prime}}  \sum_{d=1}^{D_{t^\prime}} \lambda(y_{t^\prime, d}, \xi_{t^\prime, d}(i))} \prod_{r=t}^{t^\prime-1}\left[(1-\alpha) e^{-\frac{\eta}{D_r}  \sum_{d=1}^{D_r} \lambda(y_{r, d}, \xi_{r, d}(i))}\right] \tilde{w}_{t-1}(i)\\
			= \tilde{w}_{t-1}(i) e^{-\eta  \sum_{r=t}^{t^\prime} \frac{1}{D_r}\sum_{d=1}^{D_r} \lambda(y_{r, d}, \xi_{r, d}(i))} (1-\alpha)^{(t^\prime-t)}.
		\end{multline*}
	\end{proof}
\end{lemma}

\begin{lemma} [Lemma 3 in \cite{warmuth_tracking_original}] \label{lemma:weights_fs2}
	For any outcomes and experts' predictions the following inequality on experts' weights holds on trial $t$
	\begin{equation}\label{eq:weights_fs2}
		\frac{\tilde{w}_t(i)}{w_t(j)} \ge \frac{\alpha}{N-1}, ~i \ne j.
	\end{equation}
	\begin{proof}
		We provide the proof from Lemma 3 in \cite{warmuth_tracking_original} for completeness. The share weight update (\ref{eq:fs_update}) on step $t$ is 
		\begin{equation*}
			\tilde{w}_t(i) = \frac{\alpha}{N-1} \sum_{j \ne i} w_t(j) + (1-\alpha) w_t(i).
		\end{equation*}
		By keeping only one term from the sum and dropping the second term we get the inequality (\ref{eq:weights_fs2}).
	\end{proof}
\end{lemma}

{\bf First proof of Theorem \ref{theorem:bound_fs}}
\begin{proof}
	For the proof we use Lemma \ref{lemma:AAP_bound} and follow lectures notes in \cite{kalnishkan2018ml_lectures}. Recall that superexpert $\mathcal{S}_{T, N, k, \bf{t}, \bf{e}}$ can switch between base experts $\mathcal{E}_i$ and $\mathcal{E}_j$ with probability $\alpha$. Let the initial distribution of the base experts be uniform $w_0(i) = 1/N$. The probability of switching between base experts $\mathcal{E}_i$ and $\mathcal{E}_j$ is:
	\begin{equation*}
		p_{ij} = \begin{cases}
			1-\alpha,\text{ if $i=j$}\\
			\frac{\alpha}{N-1},\text{ otherwise}
		\end{cases}.
	\end{equation*}
	
	Then the initial probability of the superexpert $\mathcal{S}_{T, N, k, \bf{t}, \bf{e}}$ with $k$ switches is:
	\begin{equation}
		w_0(\mathcal{S}_{T, N, k, \bf{t}, \bf{e}}) = \frac{1}{N} \left( \frac{\alpha}{N-1} \right)^k (1-\alpha)^{T-1-k}.
	\end{equation}
	When we substitute the base expert $\mathcal{E}_i$ with the superexpert $\mathcal{S}_{T, N, k, \bf{t}, \bf{e}}$ and the initial weights distribution $w_0(i)$ with $w_0(\mathcal{S}_{T, N, k, \bf{t}, \bf{e}})$ in (\ref{eq:AAP_bound}) we have
	\begin{multline*}
		L_T^{\textrm{average}} \le C L_T^{\textrm{average}}(\mathcal{S}_{T, N, k, \bf{t}, \bf{e}}) + \frac{C}{\eta} \ln \frac{1}{w_0(\mathcal{S}_{T, N, k, \bf{t}, \bf{e}})} \\
		\le CL_T^{\textrm{average}} (\mathcal{S}_{T, N, k, \bf{t}, \bf{e}}) + \frac{C}{\eta} \left( \ln N + k \ln \frac{N-1}{\alpha}
		+ (T - 1 - k) \ln \frac{1}{1 - \alpha} \right).
	\end{multline*}
\end{proof}

{\bf Second proof of Theorem \ref{theorem:bound_fs}}
\begin{proof}
	The step on line~6 of Protocol~\ref{aap_incremental} states that at step $t$
	\begin{equation*}
		e^{-\eta \lambda(y_{t,d}, \gamma_{t,d})/C}\ge
		\sum_{i=1}^N \tilde{w}_{t-1}^*(i)e^{-\eta \lambda(y_{t,d}, \xi_{t,d}(i))}, \forall d=1,2,\dots,D_t
	\end{equation*}
	By multiplying these equations $D_t$ times we have
	\begin{equation*}
		e^{-\eta \sum_{d=1}^{D_t}\lambda(y_{t,d}, \gamma_{t,d})/C}\ge \prod_{d=1}^{D_t}
		\sum_{i=1}^N \tilde{w}_{t-1}^*(i)e^{-\eta \lambda(y_{t,d}, \xi_{t,d}(i))}.
	\end{equation*}
	
	We apply the generalised H\"older inequality to the right-hand side of the inequality which states that on measure spaces $(S, \Sigma, \mu)$, formed by the space $S$, the $\sigma$-field $\Sigma$ and the measure $\mu$ defined on this $\sigma$-field,  for all measurable real- or complex-valued functions $f_1, \dots, f_{D_t}$ defined on $S$: $\|\prod_{d=1}^{D_t} f_d\|_r\le
	\prod_{d=1}^{D_t}\|f_d\|_{r_d}$, where $\sum_{d=1}^{D_t} 1/r_d=1/r$. This
	follows by induction from the version of the inequality given by
	\cite[Section~9.3]{loeve_vol1}.
	
	By defining a function
	$f_d = (e^{-\eta \lambda(y_{t,d}, \xi_{t,d}(1)}, \dots, e^{-\eta \lambda(y_{t,d}, \xi_{t,d}(N))})\in\mathbb{R}^N$, introducing a
	measure $\tilde{w}_{t-1}^*(i)$, $i=1,2,\ldots,N$, and letting $r_d=1$ and $r = 1/D_t$ we have
	\begin{equation*}
		e^{-\eta \sum_{d=1}^{D_t}\lambda(y_{t,d}, \gamma_{t,d})/C}\ge \prod_{d=1}^{D_t}
		\sum_{i=1}^N \tilde{w}_{t-1}^*(i) e^{-\eta \lambda(y_{t,d}, \xi_{t,d}(i))}
		\ge
		\left(\sum_{i=1}^N \tilde{w}_{t-1}^*(i) e^{-\eta \sum_{d=1}^{D_t} \lambda(y_{t,d}, \xi_{t,d}(i))/D_t}\right)^{D_t}.
	\end{equation*}
	Raising the resulting inequality to the power $1/D_t$, taking the logarithm and multiplying by $C / \eta$ we have
	\begin{multline*} 
		\frac{1}{D_t}\sum_{d=1}^{D_t}\lambda(y_{t,d}, \gamma_{t,d}) \le - \frac{C}{\eta} \ln \sum_{i=1}^N \tilde{w}_{t-1}^*(i) e^{-\eta \sum_{d=1}^{D_t} \lambda(y_{t,d}, \xi_{t,d}(i))/D_t} \\
		= -\frac{C}{\eta} \ln \frac{\sum_{i=1}^N \tilde{w}_{t-1}(i)e^{-\eta\sum_{d=1}^{D_t}\lambda(y_{t,d}, \xi_{t,d}(i))/D_t}}{\sum_{i=1}^N \tilde{w}_{t-1}(i)} 
		= -\frac{C}{\eta} \ln \frac{\sum_{i=1}^N w_t(i)}{\sum_{i=1}^N \tilde{w}_{t-1}(i)}.
	\end{multline*}
	
	Let us denote $W_t:=  \sum_{i=1}^N \tilde{w}_t(i)$. From the share weights update (\ref{eq:fs_update}) we can see that the total sum of the share weights update is equal to the total sum of the intermediate weights at time $t$
	\begin{equation*}
		\sum_{i=1}^N \tilde{w}_t(i) = (1-\alpha) \sum_{i=1}^N w_t(i) + \frac{\alpha}{N-1} \left(N \sum_{j=1}^N w_t(j) - \sum_{i=1}^N w_t(i) \right) = \sum_{i=1}^N w_t(i). 
	\end{equation*}
	Then
	\begin{multline} \label{AA_prop_delays}
		L_T^{\mathrm{average}}:=\sum_{t=1}^T \frac{1}{D_t}\sum_{d=1}^{D_t}\lambda(y_{t,d}, \gamma_{t,d})\le  -\frac{C}{\eta} \sum_{t=1}^T 
		\ln \frac{W_t}{W_{t-1}} =
		-\frac{C}{\eta} \ln \prod_{t=1}^T \frac{W_t}{W_{t-1}} 
		=  -\frac{C}{\eta} \ln W_T \le -\frac{C}{\eta} \ln \tilde{w}_T(i).
	\end{multline}
	
	Recall that superexpert $\mathcal{S}_{T, N, k, \bf{t}, \bf{e}}$ follow predictions of experts $(e_0, e_1,\dots,e_k)$ on intervals $[t_0, t_1), [t_1, t_2),\dots,[t_k, t_{k+1})$ respectively. On trial $T$ we can express the share weight as the following telescoping product
	\begin{equation*}
		\tilde{w}_T(e_k) = \tilde{w}_{t_0-1}(e_0) \frac{w_{t_1-1}(e_0)}{\tilde{w}_{t_0-1}(e_0)} \prod_{i=1}^k \left( \frac{\tilde{w}_{t_i-1}(e_i)}{w_{t_i-1}(e_{i-1})} \frac{w_{t_{i+1}-1}(e_i)}{\tilde{w}_{t_i-1}(e_i)} \right)
		\frac{\tilde{w}_T(e_k)}{w_{t_{k+1}-1}(e_k)}.
	\end{equation*}
	By applying Lemma \ref{lemma:weights_fs1} and Lemma \ref{lemma:weights_fs2} we get 
	\begin{multline*}
		\tilde{w}_T(e_k) \ge \tilde{w}_{t_0-1}(e_0) \left(\frac{\alpha}{N-1} \right)^k 
		\prod_{i=0}^k \left[ e^{-\eta \sum_{r=t_i}^{t_{i+1}-1} \frac{1}{D_r} \sum_{d=1}^{D_r}\lambda(y_{r,d}, \xi_{r,d}(e_i))} (1-\alpha)^{(t_{i+1}-t_i-1)} \right]
		\frac{\tilde{w}_T(e_k)}{w_{t_{k+1}-1}(e_k)}.
	\end{multline*}
	By definition,  the cumulative average loss of superexpert $\mathcal{S}_{T, N, k, \bf{t}, \bf{e}}$ is $\sum_{i=0}^k \sum_{r=t_i}^{t_{i+1}-1}\frac{1}{D_r} \sum_{d=1}^{D_r} \lambda(y_{r,d}, \xi_{r,d}(e_i))$, and $ \frac{\tilde{w}_T(e_k)}{w_{t_{k+1}-1}(e_k)}$ is equal to one, and \\$\sum_{i=0}^k \left(t_{i+1} - t_i - 1 \right) = t_{k+1} - t_0 - k = T - k - 1$. Then we have
	\begin{equation*}
		\tilde{w}_T(e_k) \ge \frac{1}{N} e^{-\eta L_T^{\mathrm{average}}(\mathcal{S}_{T, N, k, \bf{t}, \bf{e}})} (1-\alpha)^{T-k-1} \left(\frac{\alpha}{N-1} \right)^k.
	\end{equation*}
	The proof is completed by putting this expression in (\ref{AA_prop_delays}).
\end{proof}

\begin{lemma}[Lemma 4 in \cite{warmuth_tracking_original}] \label{lemma:lemma4_herbster}
	If $\beta > 0$ and $r \in [0, 1]$, then $\beta^r \le 1 - (1-\beta)r$ and $1-(1-\beta)^r \ge \beta r$. 
\end{lemma}

\begin{lemma}[Lemma 5 in \cite{warmuth_tracking_original}] \label{lemma:lemma5_herbster}
	Given $b, c \in [0, 1], d \in (0, 1]$ and $c+d \ge 1$, then $b^c(c+db^d) \ge b$.
\end{lemma}

\begin{lemma} \label{lemma:lemma6_herbster}
	\begin{equation}\label{eq:lemma6}
		\tilde{w}_t(i) \ge \begin{cases}
			\tilde{w}_{t-1}(i) \left[e^{-\eta} (1-\alpha) \right]^{\frac{1}{D_t} \sum_{d=1}^{D_t} \lambda(y_{t,d}, \xi_{t,d}(i))} \text{\hspace{3.9cm} \bf{(a)}}\\
			\tilde{w}_{t-1}(j) e^{-\frac{\eta}{D_t} \sum_{d=1}^{D_t} \lambda(y_{t,d}, \xi_{t,d}(j))}  \frac{\alpha}{(N-1)D_t} \sum_{d=1}^{D_t} \lambda(y_{t,d}, \xi_{t,d}(j)), ~j \ne i \text{\hspace{0.1cm}  \bf{(b)}}
		\end{cases}
	\end{equation}
\end{lemma}
\begin{proof}
	The proof is similar to the proof of Lemma 6 from \cite{warmuth_tracking_original}. By combining (\ref{eq:vs_update}) and line 9 from Protocol \ref{aap_incremental} we have
	\begin{multline} \label{eq:share_weight_variable}
		\tilde{w}_t(i) =
		\tilde{w}_{t-1}(i) e^{-\frac{\eta}{D_t} \sum_{d=1}^{D_t} \lambda(y_{t,d}, \xi_{t,d}(i))} 
		(1-\alpha)^{\frac{1}{D_t} \sum_{d=1}^{D_t} \lambda(y_{t,d}, \xi_{t,d}(i))} \\
		+ \frac{1}{N-1}\sum_{j \ne i} \tilde{w}_{t-1}(j) e^{-\frac{\eta}{D_t} \sum_{d=1}^{D_t} \lambda(y_{t,d}, \xi_{t,d}(j))} \left[1 - (1-\alpha)^{\frac{1}{D_t} \sum_{d=1}^{D_t} \lambda(y_{t,d}, \xi_{t,d}(j))} \right].
	\end{multline}
	Eq. (\ref{eq:lemma6}(a)) is obtained by dropping the second part of (\ref{eq:share_weight_variable}).  By Lemma \ref{lemma:lemma4_herbster},  we have that
	\begin{equation*}
		1 - (1-\alpha)^{\frac{1}{D_t} \sum_{d=1}^{D_t} \lambda(y_{t,d}, \xi_{t,d}(j))} \ge \frac{\alpha}{D_t} \sum_{d=1}^{D_t} \lambda(y_{t,d}, \xi_{t,d}(j)).
	\end{equation*}
	We obtain (\ref{eq:lemma6}(b)) by putting this expression in (\ref{eq:share_weight_variable}) and dropping all the terms apart from one in the second part of the equation.
\end{proof}
\begin{lemma}\label{lemma:lemma7_herbster}
	The share weight of expert $i$ from the end of trial $t$ to the end of trial $t^\prime$ is reduced by no more than a factor $\left[e^{-\eta} (1-\alpha) \right]^{{\sum_{r=t+1}^{t^\prime}}\frac{1}{D_r} \sum_{d=1}^{D_r} \lambda(y_{r,d}, \xi_{r,d}(i))}$, i.e.
	\begin{equation*}
		\frac{\tilde{w}_{t^\prime}(i)}{\tilde{w}_t(i)} \ge \left[e^{-\eta} (1-\alpha) \right]^{{\sum_{r=t+1}^{t^\prime}}\frac{1}{D_r} \sum_{d=1}^{D_r} \lambda(y_{r,d}, \xi_{r,d}(i))}.
	\end{equation*}
	\begin{proof}
		The proof is by applying (\ref{eq:lemma6}(a)) iteratively on trials $t+1, \dots t^\prime$.
	\end{proof}
\end{lemma}
\begin{lemma}\label{lemma:lemma8_herster}
	For any distinct experts $p$ and $q$, if $\sum_{r=t}^{t^\prime-1} \frac{1}{D_r}\sum_{d=1}^{D_r} \lambda(y_{r,d}, \xi_{r,d}(p)) < 1$ and  $1 \le \sum_{r=t}^{t^\prime} \frac{1}{D_r}\sum_{d=1}^{D_r} \lambda(y_{r,d}, \xi_{r,d}(p)) < 2$ then we may lower bound the share weight of expert $q$ on trial $t^\prime$ by
	\begin{equation*}
		\tilde{w}_{t^\prime}(q) 
		\ge \tilde{w}_{t-1}(p) \left[ \frac{\alpha}{N-1} e^{-\eta} (1-\alpha)\right] \left[e^{-\eta} (1-\alpha) \right]^{\sum_{r=t}^{t^\prime} \frac{1}{D_r} \sum_{d=1}^{D_r} \lambda(y_{r,d}, \xi_{r,d}(q))}.
	\end{equation*}
	\begin{proof}
		Let $a_s$ be the weight transferred from expert $p$ to expert $q$ at trial $s$, $t \le s \le t^\prime$, and the total weight transferred through trials $t,\dots, t^\prime$ is $A = \sum_{s=t}^{t^\prime} a_s$. By Lemma \ref{lemma:lemma7_herbster} we have that
		\begin{equation*}
			\tilde{w}_{t^\prime}(q) \ge \sum_{s=t}^{t^\prime} a_s \left[e^{-\eta} (1-\alpha) \right]^{\sum_{r=s}^{t^\prime} \frac{1}{D_r} \sum_{d=1}^{D_r} \lambda(y_{r,d}, \xi_{r,d}(q))}.
		\end{equation*}
		Note, that 
		\begin{equation*}
			\sum_{r=s}^{t^\prime} \frac{1}{D_r}\sum_{d=1}^{D_r} \lambda(y_{r,d}, \xi_{r,d}(q)) \le \sum_{r=t}^{t^\prime} \frac{1}{D_r} \sum_{d=1}^{D_r} \lambda(y_{r,d}, \xi_{r,d}(q)), ~ s \ge t.
		\end{equation*}
		Then we can bound the share weight of expert $q$ at trial $t^\prime$ as
		\begin{equation} \label{eq:weight_q}
			\tilde{w}_{t^\prime}(q) \ge
			A \left[e^{-\eta} (1-\alpha) \right]^{\sum_{r=t}^{t^\prime} \frac{1}{D_r} \sum_{d=1}^{D_r} \lambda(y_{r,d}, \xi_{r,d}(q))}.
		\end{equation}
		Denote $l_s = \frac{1}{D_s} \sum_{d=1}^{D_s} \lambda(y_{s, d}, \xi_{s, d}(p))$ be the average loss of expert $p$ at trial $s$. From the assumptions of this lemma, we have that 
		$\sum_{s=t}^{t^\prime-1} l_s < 1$ and $1 \le \sum_{s=t}^{t^\prime} l_s < 2$.
		
		Applying Lemma \ref{lemma:lemma7_herbster} on trials $t,\dots, s-1$ to expert $p$ and Lemma  \ref{lemma:lemma6_herbster}(b) on trial $s$ we have
		\begin{equation*}
			a_s \ge \tilde{w}_{t-1}(p) \frac{\alpha}{N-1} l_s e^{-\eta \sum_{r=t}^s l_r} (1-\alpha)^{\sum_{r=t}^{s-1} l_r}.
		\end{equation*}
		Then the total weight is
		\begin{multline*}
			A = \sum_{s=t}^{t^\prime} a_s \ge \tilde{w}_{t-1}(p) \frac{\alpha}{N-1} \sum_{s=t}^{t^\prime} \left(l_s e^{-\eta \sum_{r=t}^s l_r} (1-\alpha)^{\sum_{r=t}^{s-1} l_r} \right)\\
			= \tilde{w}_{t-1}(p) \frac{\alpha}{N-1} \sum_{s=t}^{t^\prime-1} \left(l_s e^{-\eta \sum_{r=t}^s l_r} (1-\alpha)^{\sum_{r=t}^{s-1} l_r} \right) 
			+ \tilde{w}_{t-1}(p) \frac{\alpha}{N-1} l_{t^\prime} e^{-\eta \sum_{r=t}^{t^\prime} l_r} (1-\alpha)^{\sum_{r=t}^{t^\prime-1} l_r} \\
			\ge \tilde{w}_{t-1}(p) \frac{\alpha}{N-1} \sum_{s=t}^{t^\prime-1} l_s e^{-\eta \sum_{s=t}^{t^\prime-1} l_s} (1-\alpha)
			+ \tilde{w}_{t-1}(p) \frac{\alpha}{N-1} l_{t^\prime} e^{-\eta \sum_{s=t}^{t^\prime} l_s} (1-\alpha) \\
			= \tilde{w}_{t-1}(p) \frac{\alpha(1-\alpha)}{N-1} \left[\sum_{s=t}^{t^\prime-1} l_s e^{-\eta \sum_{s=t}^{t^\prime-1} l_s} + l_{t^\prime} e^{-\eta \sum_{s=t}^{t^\prime} l_s} \right].
		\end{multline*}
		The last inequality follows from inequalities
		\begin{equation*}
			e^{-\eta\sum_{r=t}^s l_r} \ge e^{-\eta\sum_{r=t}^{t^\prime -1} l_r},~ s \le t^\prime-1 \hspace{1cm}\text{and}
		\end{equation*}
		\begin{equation*}
			(1-\alpha)^{\sum_{r=t}^{s-1} l_r} \ge (1-\alpha)^{\sum_{r=t}^{t^\prime-2} l_r}
			\ge (1-\alpha)^{\sum_{r=t}^{t^\prime-1} l_r} \ge 1-\alpha,~ s \le t^\prime-1,
		\end{equation*}
		because $\sum_{r=t}^{t^\prime-1} l_r < 1$ by the assumptions of this lemma.
		By substituting $b = e^{-\eta}$, $c = \sum_{s=t}^{t^\prime-1} l_s < 1$ and $d = l_{t^\prime}$ and applying Lemma \ref{lemma:lemma5_herbster} we have
		\begin{multline*}
			A \ge \tilde{w}_{t-1}(p) \frac{\alpha(1-\alpha)}{N-1} \left[c b^c + d b^{c+d} \right] = \tilde{w}_{t-1}(p) \frac{\alpha(1-\alpha)}{N-1} b^c (c + db^d)\\
			\ge \tilde{w}_{t-1}(p) \frac{\alpha(1-\alpha)}{N-1} b = \tilde{w}_{t-1}(p) \frac{\alpha}{N-1} e^{-\eta} (1-\alpha).
		\end{multline*}
		Putting this inequality in (\ref{eq:weight_q}) completes the proof.
	\end{proof}
\end{lemma}

\begin{lemma} \label{lemma:collapsed}
	For any superexpert $\mathcal{S}_{T, N, k, \bf{t}, \bf{e}}$ there exists a collapsed superexpert $\mathcal{S}_{T, N, k^\prime, \bf{t^\prime}, \bf{e^\prime}}$, such that the total average loss at any segment (except the initial) of the expert associated with the previous segment is at least one, and the cumulative average loss of this collapsed superexpert exceeds the cumulative average loss of the original superexpert by no more than $k-k^\prime$, i.e.
	\begin{equation} \label{eq:collapsed1}
		\sum_{r = t_i^\prime}^{t^\prime_{i+1}-1} \frac{1}{D_r} \sum_{d=1}^{D_r} \lambda(y_{r, d}, \xi_{r, d}(e^\prime_{i-1}) \ge 1, ~\forall i: 1 \le i \le k^\prime \text{\hspace{0.5cm} and} 
	\end{equation}
	\begin{equation} \label{eq:collapsed2}
		L_T^{\textrm{average}}(\mathcal{S}_{T, N, k^\prime, \bf{t^\prime}, \bf{e^\prime}}) \le 
		L_T^{\textrm{average}}(\mathcal{S}_{T, N, k, \bf{t}, \bf{e}}) + k - k^\prime.
	\end{equation}
	\begin{proof}
		The Lemma is the analogue of Lemma 9 in \cite{warmuth_tracking_original} for the cumulative average losses.  If at interval $[t_i, t_{i+1})$ the total average loss of expert $e_{i-1}$ associated with interval $[t_{i-1}, t_i)$ is less than one, then these two intervals are merged to interval $[t_{i-1}, t_{i+1})$,  and this segment is associated with expert $e_{i-1}$. We continue to do this until (\ref{eq:collapsed1}) holds. Note, that if we merge segments $i-1$ and $i$ then the total loss of expert $e_{i-1}$ on segment $i$ is less than one. Therefore, the total average loss of the new collapsed superexpert is increased by no more than one. As the total number of these collapses is $k-k^\prime$, the total average loss of the collapsed superexpert exceeds the total average loss of the original superexpert by no more than $k - k^\prime$.
	\end{proof}
\end{lemma}

{\bf Proof of Theorem \ref{theorem:bound_vs}}
\begin{proof}
	We consider an arbitrary superexpert $\mathcal{S}_{T, N, k, \bf{t}, \bf{e}}$.  If for this superexpert property (\ref{eq:collapsed1}) does not hold,  then by Lemma \ref{lemma:collapsed} we can replace it with a collapsed superexpert $\mathcal{S}_{T, N, k^\prime, \bf{t^\prime}, \bf{e^\prime}}$,  for which is holds. Now as property (\ref{eq:collapsed1}) holds there exists step $q_i$ in $i$th segment ($1 \le q_i \le k^\prime$), such that $\sum_{r=t_i^\prime-1}^{q_i^\prime-1} \frac{1}{D_r}\sum_{d=1}^{D_r} \lambda(y_{r,d}, \xi_{r,d}(e^\prime_{i-1})) < 1$ and  $1 \le \sum_{r=t_i^\prime-1}^{q_i^\prime} \frac{1}{D_r}\sum_{d=1}^{D_r} \lambda(y_{r,d}, \xi_{r,d}(e^\prime_{i-1})) < 2$. On trial $T$ we can express the share weight of a collapsed superexpert $\mathcal{S}_{T, N, k,^\prime \bf{t^\prime}, \bf{e^\prime}}$ as the following telescoping product
	\begin{equation*}
		\tilde{w}_T(e^\prime_{k^\prime}) = \tilde{w}_{t^\prime_0-1}(e^\prime_0) \frac{\tilde{w}_{t^\prime_1-1}(e^\prime_0)}{\tilde{w}_{t^\prime_0-1}(e^\prime_0)} \prod_{i=1}^{k^\prime} \left( \frac{\tilde{w}_{q_i}(e^\prime_i)}{\tilde{w}_{t^\prime_i-1}(e^\prime_{i-1})} \frac{\tilde{w}_{t^\prime_{i+1}-1}(e^\prime_i)}{\tilde{w}_{q_i}(e^\prime_i)} \right).
	\end{equation*}
	By applying Lemma \ref{lemma:lemma7_herbster} and Lemma \ref{lemma:lemma8_herster} we get 
	\begin{multline} \label{eq:vs_last}
		\tilde{w}_T(e^\prime_{k^\prime}) \ge \tilde{w}_{t_0^\prime - 1}(e_0^\prime) \left[ e^{-\eta}(1-\alpha) \right]^{\sum_{r=t_0^\prime}^{t_1^\prime-1} \frac{1}{D_r} \lambda(y_{r,d}, \xi_{r, d}(e_0^\prime))} \\
		\cdot \prod_{i=1}^{k^\prime} \left( \left[ \frac{\alpha}{N-1} e^{-\eta} (1-\alpha)\right] \left[e^{-\eta} (1-\alpha) \right]^{\sum_{r=t_i^\prime}^{t_{i+1}^\prime -1} \frac{1}{D_r} \sum_{d=1}^{D_r} \lambda(y_{r,d}, \xi_{r,d}(e_i^\prime))} \right) \\
		= \tilde{w}_{t_0^\prime - 1}(e_0^\prime) \left[e^{-\eta} (1-\alpha) \right]^{ \sum_{r=t_0^\prime}^{t_{k^\prime+1}^\prime -1} \frac{1}{D_r}\sum_{d=1}^{D_r} \lambda(y_{r,d}, \xi_{r,d}(e_i^\prime)) + k^\prime} 
		\left(\frac{\alpha}{N-1} \right)^{k^\prime} \\
		= \tilde{w}_{t_0^\prime - 1}(e_0^\prime) \left[e^{-\eta} (1-\alpha) \right]^{L_T^{\textrm{average}}(\mathcal{S}_{T, N, k^\prime, \bf{t^\prime}, \bf{e^\prime}}) + k^\prime} \left(\frac{\alpha}{N-1} \right)^{k^\prime} \\
		\ge \frac{1}{N}\left[e^{-\eta} (1-\alpha) \right]^{L_T^{\textrm{average}}(\mathcal{S}_{T, N, k, \bf{t}, \bf{e}}) + k} \left(\frac{\alpha}{N-1} \right)^{k}.
	\end{multline}
	The last inequality follows from Lemma \ref{lemma:collapsed}.  
	
	Because Variable-share and Fixed-share use the same inequality to find their predictions (line~6 of Protocol~\ref{aap_incremental}) we can show similar to the proof of Theorem \ref{theorem:bound_fs} that (\ref{AA_prop_delays}) holds 
	\begin{equation*}
		L_T^{\mathrm{average}} \le -\frac{C}{\eta} \ln \tilde{w}_T(i).
	\end{equation*}
	The proof is completed by substituting $\tilde{w}_T(i)$ with $\tilde{w}_T(e^\prime_{k^\prime}) $ from (\ref{eq:vs_last}).
\end{proof}

{\bf Proof of Corollary \ref{cor:bound_fs}}
\begin{proof}
	We follow the proof of Corollary 2 in \cite{warmuth_tracking_original}. We re-write (\ref{eq:bound_fs}) as:
	\begin{equation} \label{eq:cor_fs2}
		L_T^{\textrm{average}} \le C L_T^{\textrm{average}} (\mathcal{S}_{T, N, k, \bf{t}, \bf{e}}) + \frac{C}{\eta} \left( \ln N + f(\alpha)   \right),
	\end{equation}
	where
	\begin{equation*}
		f(\alpha) = k \ln \frac{N-1}{\alpha} + (T-1-k) \ln \frac{1}{1-\alpha}.
	\end{equation*}
	By taking the partial derivative with respect to $\alpha$ and setting it to zero we have:
	\begin{equation*}
		\frac{\partial f(\alpha)}{\partial \alpha} = -\frac{k}{\alpha} + \frac{T-1-k}{1-\alpha} = 0.
	\end{equation*}
	We have that $\alpha = \frac{k}{T-1}$ achieves zero of the first partial derivative. The second partial derivative with respect to $\alpha$ is greater than zero and therefore it minimises the regret of bound (\ref{eq:bound_fs}):
	\begin{equation*}
		\frac{\partial^2 f(\alpha)}{\partial \alpha^2} = \frac{k}{\alpha^2} + \frac{T-1-k}{(1-\alpha)^2} \ge 0.
	\end{equation*}
	Assume that the total number of steps is upper bounded by $T \le \hat{T}$ and the number of switches is lower bounded by $k \ge \hat{k}$, and by putting $\hat{\alpha} = \frac{\hat{k}}{\hat{T}-1}$ in (\ref{eq:cor_fs2}) we have:
	\begin{multline} \label{eq:cor_ineq1}
		L_T^{\textrm{average}} \le CL_T^{\textrm{average}} (\mathcal{S}_{T, N, k, \bf{t}, \bf{e}}) 
		+ \frac{C}{\eta} \left( \ln N + k \ln (N-1) + k \ln \frac{\hat{T}-1}{\hat{k}} + (T - 1 - k) \ln \Big(1 + \frac{\hat{k}}{\hat{T}-1-\hat{k}} \Big) \right).
	\end{multline}
	We bound the last term using the inequality $\ln (1+x) \le x$:
	\begin{equation} \label{eq:cor_ineq2}
		(T - 1 - k) \ln \Big(1 + \frac{\hat{k}}{\hat{T}-1-\hat{k}} \Big) \le \hat{k} \frac{T - 1 - k}{\hat{T}-1-\hat{k}} \le \hat{k}.
	\end{equation}
	The last inequality comes the bounds on the number of steps and the number of switches $T \le \hat{T}$ and $k \ge \hat{k}$, and therefore $T-k-1 \le \hat{T}-\hat{k}-1$.
	Putting (\ref{eq:cor_ineq2}) in (\ref{eq:cor_ineq1}) completes the proof.
\end{proof}

{\bf Proof of Corollary \ref{cor:bound_vs}}
\begin{proof}
	The proof is following the one of Corollary 3 in \cite{warmuth_tracking_original}. We re-write (\ref{eq:bound_vs}) as
	\begin{equation} \label{eq:bound_vs2}
		L_T^{\textrm{average}} \le CL_T^{\textrm{average}} (\mathcal{S}_{T, N, k, \bf{t}, \bf{e}}) 
		+ \frac{C}{\eta} \left[ \ln N + k(\ln(N-1) + \eta) + f(\alpha) \right],
	\end{equation}
	where
	\begin{equation} \label{eq:f_alpha_vs}
		f(\alpha) = \ln \frac{1}{1-\alpha} L_T^{\textrm{average}} (\mathcal{S}_{T, N, k, \bf{t}, \bf{e}}) + k \ln \frac{1}{\alpha(1-\alpha)}.
	\end{equation}
	By taking the partial derivative with respect to $\alpha$ and setting it to zero we have:
	\begin{equation*}
		\frac{\partial f(\alpha)}{\partial \alpha} = \frac{1}{1-\alpha} L_T^{\textrm{average}} (\mathcal{S}_{T, N, k, \bf{t}, \bf{e}}) + k \frac{2\alpha-1}{\alpha(1-\alpha)} = 0.
	\end{equation*}
	Assume that the total number of steps is upper bounded by $T \le \hat{T}$ and the number of switches is lower bounded by $k \ge \hat{k}$ , then $\hat{\alpha} = \frac{\hat{k}}{2\hat{k} +\hat{L} }$ achieves zero of the first partial derivative. The second partial derivative with respect to $\alpha$ is greater than zero and therefore it minimises the regret of bound (\ref{eq:bound_vs}):
	\begin{equation*}
		\frac{\partial^2 f(\alpha)}{\partial \alpha^2} = \frac{1}{(1-\alpha)^2} L_T^{\textrm{average}} (\mathcal{S}_{T, N, k, \bf{t}, \bf{e}}) + k \frac{\alpha^2 + (1-\alpha)^2}{\alpha^2 (1-\alpha)^2} \ge 0.
	\end{equation*}
	By putting $\hat{\alpha}$ in (\ref{eq:f_alpha_vs}) we have:
	\begin{equation} \label{eq:f_alpha_vs2}
		f(\hat{\alpha}) = \ln \left(1 + \frac{\hat{k}}{\hat{L} + \hat{k}} \right) L_T^{\textrm{average}} (\mathcal{S}_{T, N, k, \bf{t}, \bf{e}}) + k \ln \frac{(2\hat{k} + \hat{L})^2}{(\hat{k}+\hat{L}) \hat{k}} 
		= \frac{\hat{k} \hat{L}}{\hat{L} + \hat{k}} + k \ln \frac{(2\hat{k} + \hat{L})^2}{(\hat{k}+\hat{L}) \hat{k}} .
	\end{equation}
	The inequality is achieved by applying $L_T^{\textrm{average}} (\mathcal{S}_{T, N, k, \bf{t}, \bf{e}}) \le \hat{L}$ and $\ln(1+x) \le x$.
	
	If $\hat{k} \ge \hat{L}$, the first term is at most $\frac{1}{2} \hat{k}$ and the second term is upper bounded by $k \ln \frac{9}{2}$. Therefore, for $\hat{k} \ge \hat{L}$, we have that 
	\begin{equation*}
		f(\hat{\alpha}) \le \frac{1}{2} \hat{k} + k \ln \frac{9}{2}.
	\end{equation*}
	Putting this inequality in (\ref{eq:bound_vs2}) gives (\ref{eq:cor_vs2}).
	
	We can re-write ( \ref{eq:f_alpha_vs2}) as:
	\begin{equation*}
		f(\hat{\alpha}) = \frac{\hat{k} \hat{L}}{\hat{L} + \hat{k}} + k \ln \frac{(2\hat{k} + \hat{L})^2}{(\hat{k}+\hat{L}) \hat{L}} + k \ln \left( \frac {\hat{L}}{\hat{k}} \right) .
	\end{equation*}
	If $\hat{k} \le \hat{L}$, the first term is at most $\hat{k}$ and the second term is upper bounded by $k \ln \frac{9}{2}$. Therefore, for $\hat{k} \le \hat{L}$, we have that 
	\begin{equation*}
		f(\hat{\alpha}) \le \hat{k} + k \ln \frac{9}{2} + k \ln \left( \frac {\hat{L}}{\hat{k}} \right).
	\end{equation*}
	Putting this inequality in (\ref{eq:bound_vs2}) gives (\ref{eq:cor_vs1}).
\end{proof}

\end{document}